\documentclass[11pt]{article}

\usepackage[final]{acl}
\usepackage{capt-of}
\usepackage{times}
\usepackage{latexsym}
\usepackage{float}
\usepackage{microtype}
\usepackage{graphicx}
\usepackage{booktabs} 
\usepackage{wrapfig}
\usepackage{hyperref}
\usepackage[most]{tcolorbox}
\usepackage{multirow} 

\usepackage{enumitem}

\usepackage{amsmath}
\usepackage{amssymb}
\usepackage{mathtools}
\usepackage{amsthm}
\usepackage{algorithm}
\usepackage{algorithmic}
\usepackage[subrefformat=parens]{subcaption}

\usepackage[capitalize,noabbrev]{cleveref}

\theoremstyle{plain}
\newtheorem{theorem}{Theorem}

\theoremstyle{definition}

\theoremstyle{remark}

\usepackage{microtype}
\usepackage{graphicx}

\title{Validity‑Calibrated Reasoning Distillation}

\author{Khouloud Saadi , Di Wang \\
  King Abdullah University of Science and Technology \\
  Computer Science Department \\
  KAUST, Thuwal, Saudi Arabia \\
\texttt{\{khouloud.saadi,di.wang\}@kaust.edu.sa} 
  }

\begin{document}
\raggedbottom
\maketitle

\begin{abstract}
Reasoning distillation aims to transfer multi-step reasoning capabilities from large language models to smaller, efficient ones. While recent methods have shown promising gains, they typically rely on static teacher–student hierarchies and frame distillation as trajectory imitation. This is misaligned with the structure of reasoning, where intermediate steps are \emph{locally under-specified}: global correctness constrains the final answer, but does not uniquely determine each intermediate move.
We propose \emph{validity-calibrated reasoning distillation}, a framework that treats reasoning distillation as a problem of \emph{local learning-signal allocation} rather than path alignment. Instead of enforcing token-level imitation, we compare the student’s and teacher’s proposed next-step actions under the same prefix and use their \emph{relative local validity} to modulate the strength of the distillation update. This yields a dynamic, context-dependent supervision that preserves the teacher’s structural guidance while adapting update strength to local reasoning quality.
Across math, code, and instruction-following benchmarks, our method outperforms strong distillation baselines. These results indicate that effective LLM reasoning distillation is governed not by rigid trajectory imitation, but by locally calibrated allocation of learning signal.
\end{abstract}

\section{Introduction}

Reasoning has become a defining capability of modern language models
\citep{kodistillm,yangsupercorrect,pang2024iterative,shao2024deepseekmath,achiam2023gpt},
driving progress in mathematical problem solving
\citep{cobbe2021training,hendrycks2measuring}, code generation
\citep{luowizardcoder,xu2024wizardlm}, and complex instruction following.
Frontier LLMs such as GPT-4 \citep{achiam2023gpt}, DeepSeek-R1
\citep{guo2025deepseek}, and ReasonFlux \citep{yang2025reasonflux} exhibit strong
multi-step reasoning abilities but remain prohibitively expensive for widespread
deployment. In contrast, smaller language models are far more efficient and accessible,
yet they still struggle to perform reliable, context-dependent reasoning
\citep{yang2024qwen2,grattafiori2024llama}. Closing this gap is essential for
building scalable, broadly usable AI systems.
Recent efforts to improve the reasoning performance of compact models have focused on \emph{distillation}~\citep{schmidhuber1992learning,hinton2015distilling},
aiming to transfer reasoning behavior from powerful LLMs to smaller students
\citep{liu2024deepseek,yangsupercorrect}.
Most existing reasoning distillation methods~\citep{li2022explanations,fu2023specializing}
formulate supervision as a \emph{trajectory imitation} problem
\citep{magister2023teaching,liu2024deepseek,kodistillm,yangsupercorrect},
enforcing token-level alignment with a teacher rollout.
This is typically implemented through fixed teacher-centric objectives 
\citep{liu2024deepseek,kodistillm}, training the student to reproduce a specific
realization of the teacher’s reasoning process.

At a deeper level, trajectory-based reasoning distillation relies on an implicit
assumption that is rarely made explicit:  \textit{a model that is globally superior
also provides uniformly superior learning signal at each intermediate step}~\citep{li2024mode,shridhar2023distilling,li2023symbolic}.
In other words, global model quality is assumed to transfer monotonically to
local decision quality under all prefixes.
While this assumption is natural in standard prediction tasks~\citep{hinton2015distilling,jiao-etal-2020-tinybert}, where supervision
is aligned at each output position, it is fundamentally mismatched with
multi-step reasoning. Reasoning is evaluated only at the level of the final answer, whereas intermediate steps are latent variables whose correctness is often under-specified~\citep{liaoreward,achiam2023gpt}. Hence, global capability does not induce a total ordering over local
reasoning moves, and the teacher-provided learning signal can vary
substantially across decision points.
When this monotonicity breaks, uniform distillation becomes structurally miscalibrated: strong updates may be enforced where local evidence is weak, while informative local signals remain underutilized.
\begin{figure} 
  \centering
  \includegraphics[width=0.5\textwidth]{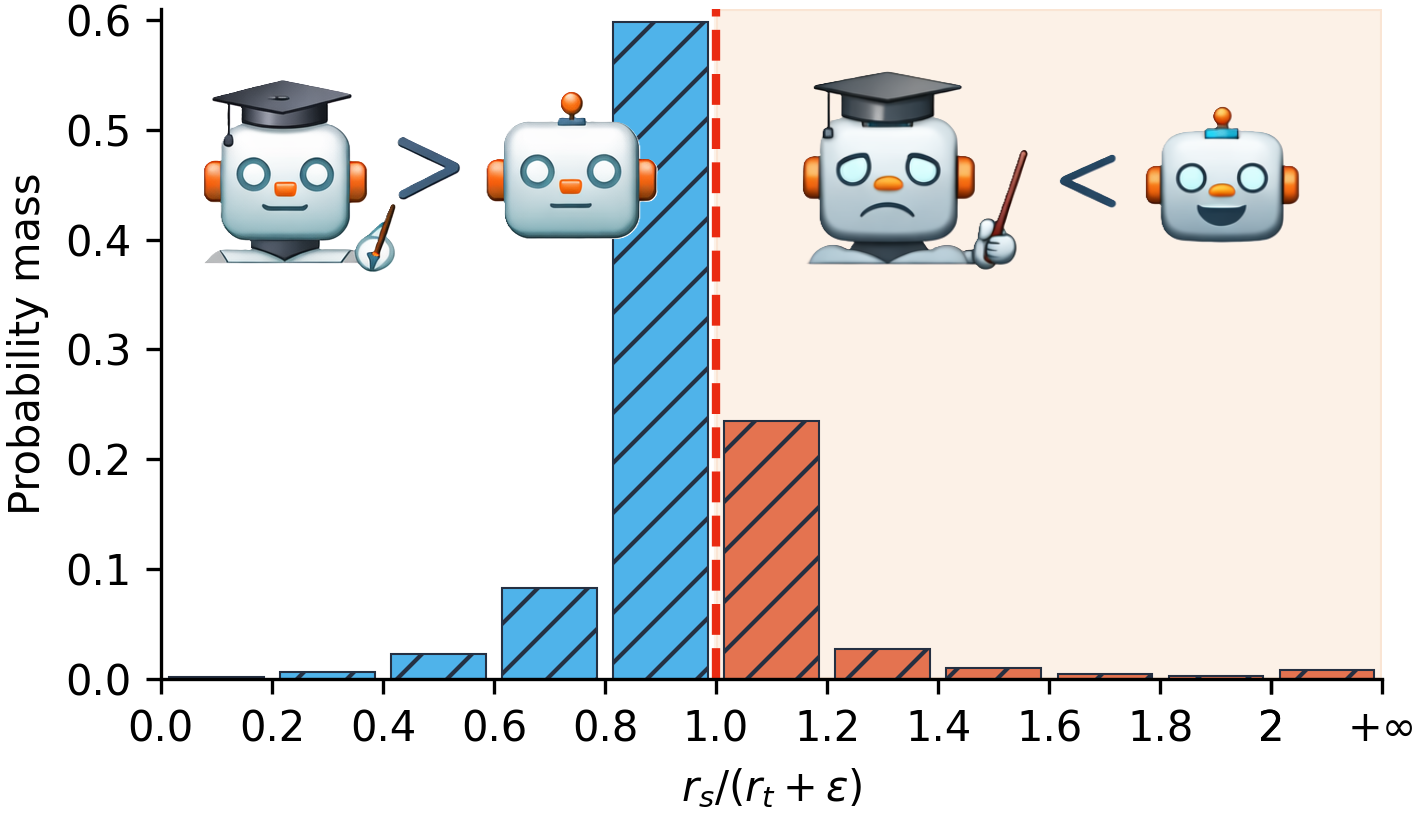}
  \vspace{-1.5 em}
  \caption{Distribution of the reward ratio $r_s / r_t$ between the student Qwen2.5-Math-1.5B ($r_s$) and the teacher Qwen2.5-Math-7B-Instruct ($r_t$). $r_s$ and $r_t$ are computed with Skywork-o1-OpenPRM-Qwen-2.5-1.5B.}
  \label{fig:ratio_dist}
  \vspace{-1.5em}
\end{figure}

To examine this mismatch, we use an auxiliary judge to compare the \emph{local validity} of teacher and student next-token proposals under identical prefixes. Figure~\ref{fig:ratio_dist} shows that the Qwen2.5-Math-1.5B student attains equal or higher local reward than the Qwen2.5-Math-7B-Instruct teacher at \textbf{28.8\% of decision points}, with qualitative examples in Appendix~\ref{app:student_better_examples}. Thus, global teacher superiority does not imply uniform local dominance. Uniform distillation overlooks this variation by coupling the teacher's structural guidance with a fixed supervision strength. A large local validity gap may make strong supervision harder for the student to exploit, whereas locally competitive states offer opportunities for consolidation. Effective reasoning distillation should therefore preserve teacher guidance while adapting update strength to the local teacher-student validity gap.

This perspective reframes the core question of reasoning distillation.
Rather than asking which trajectory should be imitated, the central challenge is
to determine \emph{how strongly the student should update at each decision
point}. We propose \textbf{Validity-Calibrated Reasoning Distillation (VCRD)}, a
framework that allocates token-level supervision based on the \emph{relative
local validity} of teacher and student proposals. At each prefix, both models
produce candidate next steps, which are evaluated under the same context by an
auxiliary judge. 
Their relative validity then scales the distillation update, yielding
three regimes: parity when the proposals have comparable local validity,
attenuation when the student exhibits a larger local gap relative to the
teacher, and amplification when the student is locally competitive with
or stronger than the sampled teacher continuation. Intuitively, this
relative validity acts as a local competence signal, allowing VCRD to
reduce learning pressure on larger-gap states while emphasizing states
that are more compatible with the student's current capability.
Importantly, VCRD does not alter the direction of supervision; it
adaptively calibrates its strength, preserving teacher guidance while
aligning learning pressure with local reasoning evidence.
Additional discussion of related work appears in Appendix~\ref{rl}. In summary, our contributions are as follows:
\vspace{-0.5 em}
\begin{itemize}
    \item We identify a previously implicit \emph{monotonicity assumption} in trajectory-based LLM reasoning distillation, namely, that global teacher superiority implies uniformly superior local learning signal, and show that this assumption is violated in multi-step reasoning.
    \vspace{- 0.5 em}
 \item We introduce VCRD, which performs matched-prefix relative-validity
calibration: teacher and student proposals are evaluated under the same state, and their relative validity is used to allocate the strength of distributional distillation across decision points.
\vspace{-0.5em}
\item We provide a theoretical motivation via a teacher-anchored KL
trust-region formulation, showing how local validity shapes the ideal
first-order policy improvement and motivates relative-validity allocation
of KD strength.
\vspace{-0.5 em}
    \item We empirically validate the importance of the \emph{amplification}
    regime, showing that VCRD yields consistent improvements across mathematical
    reasoning, code generation, and instruction-following tasks. Strong ablations further show the effectiveness of VCRD.
\end{itemize}
\vspace{-1 em}
\section{Validity-Calibrated Reasoning Distillation}
\label{sec:method}
\vspace{-0.7 em}
\begin{figure}[bt]
  \vskip 0.2in
  \begin{center}
\centerline{\includegraphics[width=1\columnwidth]
{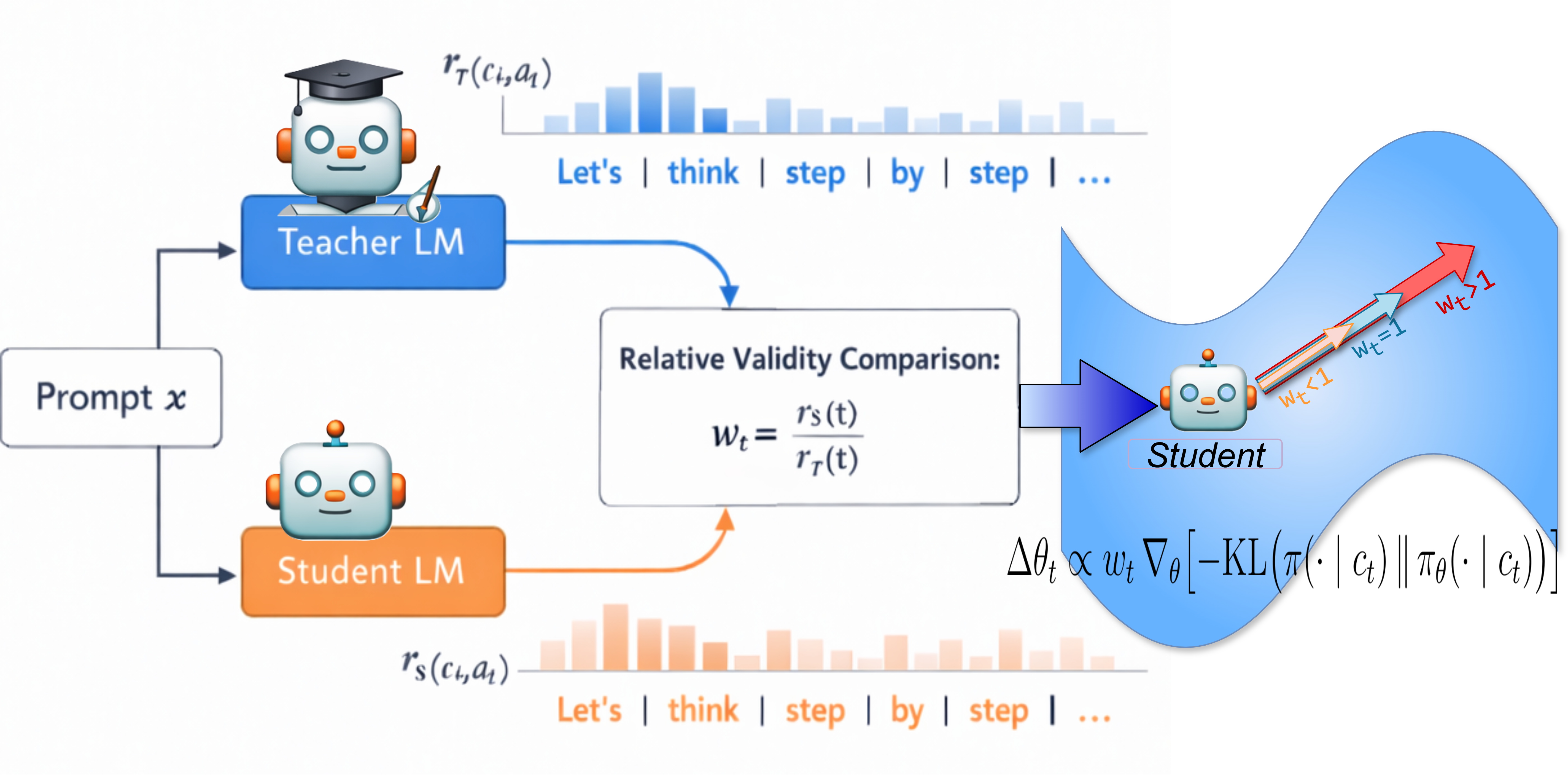}}
    \caption{Overview of VCRD.
Rather than enforcing uniform trajectory imitation, VCRD allocates
token-level learning signal based on the relative local validity of teacher and
student proposals under the same prefix.
By modulating update strength rather than direction, the method preserves
teacher guidance while adapting supervision to locally under-specified
reasoning steps.
    }
    \label{fig:method_overview}
  \end{center}
  \vspace{-3 em}
\end{figure}

Reasoning distillation differs from standard knowledge distillation in a
fundamental way.
In conventional prediction tasks, the teacher’s output defines a
well-specified target, and deviations by the student can be unambiguously
interpreted as errors~\citep{Sun2019PatientKD,jiao-etal-2020-tinybert}.
In multi-step reasoning, however, correctness constrains only the final answer:
intermediate steps are often locally under-specified, with multiple plausible
continuations coexisting under the same prefix. As a result, teacher superiority does not imply reliable supervision
at every decision point. Despite this mismatch,  reasoning KD methods~\citep{kodistillm,wangabkd} assign the teacher a
uniformly privileged role, enforcing token-level alignment along an entire
reasoning chain and implicitly treating each teacher step as the unique correct
continuation.
This conflates the \emph{direction} of supervision with its \emph{strength},
leading to miscalibrated updates when local reasoning varies across
steps.
The central question is therefore not which trajectory should be imitated, but:
\vspace{-0.5 em}
\begin{tcolorbox}[
  colback=blue!2,
  colframe=blue!40!black,
  center title
]
\centering
\emph{Given this prefix and these candidate reasoning steps, how strongly should
the student update be applied at token $t$?}
\end{tcolorbox}
To answer this, VCRD uses an auxiliary judge to evaluate teacher and student next-token proposals under the same prefix. Their \emph{relative local validity} scales the token-level distillation loss, adapting supervision strength while preserving teacher-anchored distributional guidance (Figure~\ref{fig:method_overview}).
\vspace{-0.5 em}
\subsection{Problem Formulation}
\vspace{-0.5 em}
We consider a prompt $x$ paired with an output sequence
$y = (y_1,\dots,y_T)$. A pretrained teacher $p(y \mid x)$ and a student
$q_\theta(y \mid x)$ generate tokens autoregressively, conditioning at
step $t$ on the context $c_{t-1} = (x, y_{<t})$.
During distillation, teacher and student rollouts
$y^T \sim p(\cdot \mid x)$ and $y^S \sim q_\theta(\cdot \mid x)$
induce contexts $c^T_{t-1} = (x, y^T_{<t})$ and
$c^S_{t-1} = (x, y^S_{<t})$.
For either context $c$, the teacher and student define next-token
distributions $p(\cdot \mid c)$ and $q_\theta(\cdot \mid c)$,
from which they propose tokens $a^T_t \sim p(\cdot \mid c)$ and
$a^S_t \sim q_\theta(\cdot \mid c)$ under the same prefix.
Teacher prefixes anchor learning to high-quality partial traces,
while student prefixes expose the model to its own induced state
distribution, supporting robustness at inference.
Prior work typically mixes these contexts using fixed schedules or
interpolation coefficients~\citep{kodistillm}.
We instead calibrate supervision strength in each branch using the
\emph{local validity} of teacher and student proposals under the current prefix.
\vspace{-1.0 em}
\subsection{Local Validity} \label{local_validity}
\vspace{-0.5 em}
To evaluate the quality of individual reasoning steps, we introduce an auxiliary
\emph{validity judge} $J$ that assigns a local validity score
$r(c_{t-1}, a_t) \in [0,1]$ to a proposed next token $a_t$ under prefix
$c_{t-1}$. The judge operates purely at the token level, assessing the local
coherence of the transition $(c_{t-1}, a_t)$ without considering global solution
correctness or full trajectories. At each step $t$, teacher and student produce candidate tokens $a_t^T$ and
$a_t^S$. We compare their local validity under the \emph{same} prefix.
Under the teacher prefix $c^T_{t-1}$ and the student prefix $c^S_{t-1}$, the relative validity ratios are:
\vspace{-0.5em}
\begin{equation}
\begin{aligned}
w_t^T &=
\frac{r(c^T_{t-1}, a^S_t)}
     {r(c^T_{t-1}, a^T_t) + \varepsilon},
\;
w_t^S &=
\frac{r(c^S_{t-1}, a^S_t)}
     {r(c^S_{t-1}, a^T_t) + \varepsilon}.
\end{aligned}
\label{eq:ratio_weights}
\end{equation}
where $\varepsilon$ is a small positive constant used for numerical 
stability. 
The ratios act as position-dependent learning-signal scalers. For clarity,
we use $w_t$ below to denote either $w_t^T$ or $w_t^S$, depending on the
prefix branch. They induce three regimes:
\vspace{-1 em}
\begin{itemize}
    \item \textbf{$w_t \approx 1$: parity.}
    Teacher and student proposals have similar local validity, suggesting
    that the state is already compatible with the student's current
    capability; it therefore receives the standard distillation weight.
   \vspace{-1 em}
    \item \textbf{$w_t < 1$: attenuation.}
    The student proposal is less locally valid, indicating a larger local
    capability gap. VCRD therefore reduces the learning signal assigned to
    this harder state.
   \vspace{-1 em}
    \item \textbf{$w_t > 1$: amplification.}
    The student proposal is locally competitive with the teacher, suggesting
    that the state lies within the student's current competence. VCRD therefore
    assigns greater learning signal to consolidate learning in such states.
\end{itemize}
\vspace{-1 em}
Relative validity acts as a local competence signal: student-competitive
states may be easier to exploit, while large-gap states may benefit less
from strong supervision. Since $w_t$ is detached during optimization,
it scales the local gradient magnitude without changing the
vocabulary-level direction determined by the distillation divergence.

\vspace{-0.75 em}
\subsection{Distillation Objective}
\vspace{-0.5 em}
For any prefix $c$, the teacher and student define next-token distributions
$p(\cdot \mid c)$ and $q_\theta(\cdot \mid c)$. Rather than relying on the
standard forward or reverse KL divergences, we adopt the \emph{skew KL} (SKL) and
\emph{skew reverse KL} (SRKL) objectives \citep{kodistillm}, which provide more
stable behavior when teacher and student condition on different prefixes. We first introduce the mixture distributions:
\vspace{-1 em}
\[
\begin{aligned}
m^{(\alpha)}_{p,q_\theta}(a\mid c)
&=
\alpha p(a\mid c)
+
(1-\alpha)q_\theta(a\mid c),
\\
m^{(1-\alpha)}_{p,q_\theta}(a\mid c)
&=
(1-\alpha)p(a\mid c)
+
\alpha q_\theta(a\mid c).
\end{aligned}
\]
with $\alpha\in[0,1]$ controlling the skew between teacher and student
updates. The skewed divergences are then defined:
\vspace{-1 em}
{\small
\begin{equation}
\begin{aligned}
D^{(\alpha)}_{\mathrm{SKL}}(p \,\|\, q_\theta)
&=
\mathrm{KL}\!\left(
p(\cdot\mid c)
\,\Big\|\,
m^{(\alpha)}_{p,q_\theta}(\cdot\mid c)
\right),
\\
D^{(\alpha)}_{\mathrm{SRKL}}(p \,\|\, q_\theta)
&=
\mathrm{KL}\!\left(
q_\theta(\cdot\mid c)
\,\Big\|\,
m^{(1-\alpha)}_{p,q_\theta}(\cdot\mid c)
\right).
\end{aligned}
\label{eq:skl_srkl_def}
\end{equation}
}
\vspace{0.5 em}
\vspace{-2 em}
\paragraph{Teacher-prefix supervision (LV--SKL) \& Student-prefix supervision (LV--SRKL).}
Under teacher prefixes $c^T_{t-1}$ and student prefixes $c^S_{t-1}$, we weight the SKL and the SRKL divergences, using the local-validity ratios $w_t^T$ and $w_t^S$, respectively:
\vspace{-1 em}
{\small
\begin{equation}
\begin{aligned}
\mathcal{L}_{\mathrm{LV\!-\!SKL}}
&=
\sum_{t=1}^T
w_t^T
D^{(\alpha)}_{\mathrm{SKL}}\!\big(
  p(\cdot\mid c^T_{t-1})
  \,\big\|\,
  q_\theta(\cdot\mid c^T_{t-1})
\big),
\\
\mathcal{L}_{\mathrm{LV\!-\!SRKL}}
&=
\sum_{t=1}^T
w_t^S
D^{(\alpha)}_{\mathrm{SRKL}}\!\big(
  p(\cdot\mid c^S_{t-1})
  \,\big\|\,
  q_\theta(\cdot\mid c^S_{t-1})
\big).
\end{aligned}
\label{eq:lv_skl_srkl}
\end{equation}
}
\vspace{-2 em}
\paragraph{Final objective.}
The overall loss combines both supervision regimes:
\vspace{-1 em}
\begin{equation}
\mathcal{L}
=
\lambda_T\,\mathcal{L}_{\mathrm{LV\!-\!SKL}}
+
\lambda_S\,\mathcal{L}_{\mathrm{LV\!-\!SRKL}},
\qquad
\lambda_T, \lambda_S \ge 0.
\label{eq:final_loss}
\end{equation}
This objective preserves the geometric stability benefits of KL while
allocating token-level supervision according to the locally inferred validity of
each teacher-student decision point.
For completeness,
Algorithm~\ref{alg:vcrd} in Appendix~\ref{appendB} summarizes the full training
procedure.

\vspace{- 0.5 em}
\section{Theoretical Perspective}
\label{sec:theory}
\vspace{- 0.5 em}
We provide a theoretical view that motivates VCRD.
Rather than treating distillation as trajectory imitation, we view it as
allocating learning signal across decision points.
Formal proofs appear in
Appendix~\ref{appendA}.
\vspace{- 0.5 em}
\subsection{Reasoning as Sequential Local Decision Making}
\label{subsec:seq_decision}
\vspace{-0.5 em}
We model autoregressive reasoning as a sequential process. At step $t$,
the model conditions on the prefix $c_{t}=(x,y_{\leq t})$ and selects an action
$a_t\in\mathcal{V}$. The teacher and student induce next-token policies
$\pi(a\mid c_t)=p(a\mid c_t)$ and $\pi_\theta(a\mid c_t)=q_\theta(a\mid c_t)$, respectively. A key property of reasoning tasks is \emph{local under-specification}: correctness
constrains the final answer but typically does not fix a unique continuation at
each prefix. Disagreement between teacher and student at $c_t$ therefore does not necessarily indicate an error, but may reflect multiple locally coherent moves. This
invalidates uniform token-level imitation and raises the question:
\begin{tcolorbox}[
  colback=blue!2,
  colframe=blue!40!black,
  center title
]
\centering
\emph{Given a specific prefix \(c_t\), how much learning signal should be
allocated to this decision?}
\end{tcolorbox}

To formalize this idea, consider improving the student at a fixed prefix
$c_t$. The teacher distribution $\pi(\cdot\mid c_t)$ provides a strong
structural prior, while improvement should favor actions with high local
validity $r(c_t,a)$. This naturally leads to the teacher-anchored
trust-region objective:
\begin{equation}
\label{eq:trust_region_problem}
\begin{aligned}
\max_{\tilde{\pi}}~&
\mathbb{E}_{a\sim\tilde{\pi}}[\,r(c_t,a)\,]
\\
\text{s.t.}\quad&
\mathrm{KL}\!\big(
  \tilde{\pi}(\cdot\mid c_t)
  \,\big\|\,
  \pi(\cdot\mid c_t)
\big)
\le \delta .
\end{aligned}
\end{equation}
In Appendix~\ref{subse1}, the optimal solution takes the form of an
exponentially tilted distribution:
\vspace{-1 em}
\begin{equation}
\label{eq:exp_tilt_solution}
\tilde{\pi}^\star(a\mid c_t)
\propto
\pi(a\mid c_t)\,
\exp\!\big(\eta\, r(c_t,a)\big),
\end{equation}
for a multiplier $\eta\ge0$ determined by the trust-region radius. Rather than
imitating a full trajectory, this update redistributes probability mass
\emph{within} the teacher’s support toward actions that exhibit higher local
validity. This cleanly separates structural guidance from per-step learning-signal strength, providing the foundation for VCRD.
\vspace{-0.5 em}
\subsection{From Full-Action Improvement to Practical Learning-Signal Allocation}
\vspace{-0.25 em}
The trust-region objective in Eq.~\ref{eq:trust_region_problem} specifies an
ideal validity-aware target: probability mass is reallocated within the
teacher distribution according to local validity. Computing the
exponentially tilted optimum in Eq.~\ref{eq:exp_tilt_solution}, however,
would require evaluating the validity of a large number of vocabulary
actions at every prefix, which is impractical for large-vocabulary LLMs. The structure of the optimal solution nevertheless provides useful
guidance. Expanding $\log \tilde{\pi}^{\star}(a\mid c_t)$ around the
teacher policy for sufficiently small $\eta$ gives:
\vspace{-1. em}
\begin{equation}
\begin{aligned}
\log \tilde{\pi}^{\star}(a\mid c_t)
&=
\log \pi(a\mid c_t)
+
\eta
\Big(
r(c_t,a)
\\
&\quad
-
\mathbb{E}_{a'\sim\pi(\cdot\mid c_t)}
[r(c_t,a')]
\Big)
+
\mathcal{O}(\eta^2).
\end{aligned}
\end{equation}
Thus, to first order,
$
\Delta \log \tilde{\pi}(a\mid c_t)
\propto
r(c_t,a)
-
\mathbb{E}_{a'\sim\pi(\cdot\mid c_t)}
[r(c_t,a')],
$
showing that the ideal improvement depends on an action's validity
\emph{relative to other actions under the same prefix}.
In practice, VCRD does not attempt to reconstruct this full-action
target. At each prefix, we evaluate only the teacher and student
proposals, $a_t^T$ and $a_t^S$, and use their relative validity:
\vspace{-0.5 em}
\begin{equation} \label{validity_ratio}
w(c_t)
=
\frac{r(c_t,a_t^S)}
     {r(c_t,a_t^T)+\epsilon}
\vspace{-0.5 em}
\end{equation}
\vspace{-0.0 em}
as a tractable, scale-relative signal for allocating distillation
strength. This ratio should therefore be viewed as a practical
two-proposal allocation rule motivated by the within-prefix comparison
suggested by Subsection~\ref{subsec:seq_decision}, rather than as a unique
consequence of the optimal full-action solution.

The ratio provides a local competence signal
(Subsection~\ref{local_validity}): VCRD attenuates larger-gap states
($w(c_t)<1$), preserves the standard weight at parity, and amplifies
student-competitive states ($w(c_t)>1$).
Because $w(c_t)$ is detached,
$\nabla_\theta[w(c_t)D_t]=w(c_t)\nabla_\theta D_t$
(Appendix~\ref{app:detached_weighting}); it scales the local gradient
magnitude while SKL/SRKL determines its vocabulary-level direction.
The ratio in Eq.~\ref{validity_ratio} is an empirical design choice
within this theoretically motivated framework.
Subsection~\ref{ablation} compares it with uniform, absolute-validity,
difference-based, inverse-ratio, and mean-normalized alternatives.
\vspace{-0.5 em}
\section{Experiments}
\label{sec:experiments}
\vspace{-0.5 em}
We evaluate our \textbf{ VCRD} framework across three complementary reasoning domains: (i) mathematical problem solving, (ii) code generation, and (iii) instruction following. These settings differ substantially in structure, difficulty, and the diversity of locally valid continuations, providing a comprehensive assessment of our proposed method. In our experiments, we follow a similar setup to the one outlined in~\citet{kodistillm}. Full experimental details appear in Appendix~\ref{expsetup}. In our experiments, we use a Process Reward Model (PRM) as the \emph{validity judge}, providing dense, step-level feedback on the local correctness of intermediate reasoning steps. 
Specifically, we employ Skywork-o1-OpenPRM-Qwen-2.5-1.5B~\citep{skywork2024o1}, one of the most advanced open-source PRM available at the time of our experiments. 
This PRM is built on Qwen2.5-Math-1.5B-Instruct, an instruction-tuned backbone, and its released model card reports strong performance on both mathematical and coding evaluations.\footnote{\url{https://huggingface.co/Skywork/Skywork-o1-Open-PRM-Qwen-2.5-1.5B}}.
Given a prefix $c$ and a candidate next token $a$, the PRM returns a scalar score $r(c,a)\in[0,1]$, with higher values indicating more coherent and contextually plausible reasoning. These prefix-conditioned signals enable us to compare the relative validity of the teacher’s and student’s next-step decisions under the same context and modulate the distillation strength accordingly. Sensitivity to the validity judge appears in Appendix~\ref{app:validity_robustness}, showing that VCRD achieves nearly identical performance with 1.5B and 7B PRMs, thereby confirming its robustness to different PRM sizes. Moreover, controlled corruption of the
relative teacher-student validity ordering leads to gradual performance
degradation, indicating that VCRD benefits from informative relative
validity signals.
\vspace{-0.75 em}
\subsection{Mathematical Reasoning}
\vspace{-0.5 em}
\textbf{Setup.} We evaluate VCRD on various math‑reasoning tasks: GSM8K~\citep{cobbe2021training}, MATH~\citep{hendrycks2measuring}, SVAMP~\citep{patel2021nlp}, Minerva~\citep{lewkowycz2022solving}, Gaokao~\citep{zhang2023evaluating}, Olympiad~\citep{he2024olympiadbench}, SAT‑Math~\citep{zhong2024agieval}, CMATH~\citep{wei2023cmath}, AMC23~\citep{aimo2024amc}, and AIME24~\citep{maa2025aime}. We use the same setup as~\citep{kodistillm} with two teacher–student pairs: Qwen2-\allowbreak Math-\allowbreak 7B-\allowbreak Instruct$\rightarrow$Qwen2-\allowbreak Math-\allowbreak 1.5B and Qwen2.5-\allowbreak Math-\allowbreak 7B-\allowbreak Instruct$\rightarrow$Qwen2.5-\allowbreak Math-\allowbreak 1.5B. \\
\textbf{Results.} 
Tables~\ref{tab:qwen2_math_full} and~\ref{tab:qwen25_math_full} show that VCRD consistently improves math-reasoning performance across both the Qwen2-Math and Qwen2.5-Math settings. In the Qwen2-Math 7B→1.5B configuration, VCRD reaches an average Pass@1 of 56.06, outperforming the strongest baseline DistilLLM (54.47). The improvements are even larger for Qwen2.5-Math 7B→1.5B, where VCRD attains 59.96 (+2.09 over DistillLM-2). Remarkably, the VCRD-distilled 1.5B student approaches the average performance of its 7B teacher in the Qwen2.5 setting despite being nearly five times smaller. These results reflect VCRD’s core principle: relative-validity weighting emphasizes reasoning states that the student can handle effectively, while reducing learning pressure on harder states where the gap between the student and the teacher is high.

\begin{table*}[t]
\centering
\vspace{-0.5 em}
\resizebox{1\linewidth}{!}{%
\begin{tabular}{lcccccccc|c}
\toprule
\multicolumn{10}{c}{\textbf{Qwen2-Math-7B-Inst ($\mathcal{M}_T$) $\rightarrow$ Qwen2-Math-1.5B ($\mathcal{M}_S$)}} \\
\midrule
\textbf{Method}
& \textbf{GSM8K}
& \textbf{MATH}
& \textbf{SVAMP}
& \textbf{Minerva}
& \textbf{Gaokao}
& \textbf{Olympiad}
& \textbf{AMC23}
& \textbf{SAT-Math}
& \textbf{AVG.} \\
\midrule
$\mathcal{M}_T$
& 88.40 & 74.90 & 94.30 & 30.90 & 64.70 & 37.60 & 60.00 & 100.0 & 68.85 \\

$\mathcal{M}_S$
& 26.20 & 20.50 & 21.70 & 6.60 & 16.10 & 7.40 & 10.00 & 81.20 & 23.71 \\
\midrule
KD
& 80.10 & 61.00 & 84.30 & 26.50 & 51.90 & 23.90 & 27.50 & 40.60 & 49.48 \\

GKD
& 81.00 & 61.10 & 84.70 & 25.00 & 52.70 & 23.70 & 25.00 & 78.10 & 53.91 \\

DistilLLM
& 80.90 & 61.00 & 85.80 & 27.20 & 50.90 & 23.70 & 37.50 & 68.80 & 54.47 \\

ABKD
& 81.50 & 61.30 & 85.50 & 24.30 & 51.90 & 22.40 & 32.50 & 43.80 & 50.40 \\

DistillLM-2
& 81.90 & 61.50 & 85.00 & 23.90 & 54.00 & 24.40 & 30.00 & 65.60 & 53.29 \\
\midrule
\textbf{VCRD}
& 81.00 & 61.80 & 85.90 & 25.70 & 53.20 & 24.70 & 35.00 & 81.20 & \textbf{56.06} \\
\bottomrule
\vspace{-2 em}
\end{tabular}

}
\caption{
Pass@1 results for distilling Qwen2-Math-7B-Instruct ($\mathcal{M}_T$) into Qwen2-Math-1.5B ($\mathcal{M}_S$) on eight math-reasoning benchmarks. AVG is the average over all datasets.
}
\label{tab:qwen2_math_full}
\vspace{-1 em}
\end{table*}

\begin{table*}[t]
\centering
\resizebox{1\linewidth}{!}{%
\begin{tabular}{lccccccc|c}
\toprule
\multicolumn{9}{c}{\textbf{Qwen2.5-Math-7B-Inst ($\mathcal{M}_T$) $\rightarrow$ Qwen2.5-Math-1.5B ($\mathcal{M}_S$)}} \\
\midrule
\textbf{Method}
& \textbf{GSM8K}
& \textbf{MATH}
& \textbf{Gaokao}
& \textbf{Olympiad}
& \textbf{CMATH}
& \textbf{AIME24}
& \textbf{AMC23}
& \textbf{AVG.} \\
\midrule

$\mathcal{M}_T$
& 85.10 & 76.30 & 66.50 & 37.00 & 89.80 & 16.70 & 60.00 & 61.63 \\

$\mathcal{M}_S$
& 79.10 & 51.10 & 42.10 & 16.90 & 62.00 & 3.30 & 22.50 & 39.57 \\
\midrule

KD
& 85.50 & 73.70 & 62.90 & 32.40 & 88.20 & 10.00 & 47.50 & 57.17 \\

GKD
& 85.50 & 74.10 & 60.00 & 33.30 & 88.80 & 6.70 & 50.00 & 56.91 \\

DistilLLM
& 84.90 & 74.40 & 60.30 & 33.50 & 89.30 & 13.30 & 47.50 & 57.60 \\

ABKD
& 84.20 & 74.30 & 61.80 & 33.20 & 88.50 & 6.70 & 47.50 & 56.60 \\

DistillLM-2
& 85.70 & 73.20 & 61.60 & 34.10 & 89.70 & 13.30 & 47.50 & 57.87 \\
\midrule

\textbf{VCRD}
& 85.00 & 73.70 & 62.10 & 35.40 &90.20& 13.30& 60.00 & \textbf{59.96} \\

\bottomrule
\vspace{-2 em}
\end{tabular}
}
\caption{
Pass@1 results for distilling Qwen2.5-Math-7B-Instruct ($\mathcal{M}_T$) into Qwen2.5-Math-1.5B ($\mathcal{M}_S$) across seven math-reasoning benchmarks. AVG denotes the average score across all datasets.
}
\vspace{-0.5 em}
\label{tab:qwen25_math_full}
\vspace{-1 em}
\end{table*}

\vspace{-0.5 em}
\subsection{Code Generation}
\vspace{-0.5 em}
\textbf{Setup.} We evaluate VCRD with Qwen2.5-Coder-7B-Instruct$\rightarrow$Qwen2.5-Coder-1.5B and Qwen2.5-Coder-14B-Instruct$\rightarrow$Qwen2.5-Coder-7B-Instruct on HumanEval (HE), HE+~\citep{chen2021evaluating}, MBPP, and MBPP+~\citep{austin2021program}. We train on WizardCoder~\citep{luowizardcoder}. \\
\textbf{Results.}
Table~\ref{tab:heval_mbpp_qwen25_all} shows that VCRD consistently improves student code-generation performance across both Qwen2.5-Coder settings. In the 7B$\rightarrow$1.5B setting, VCRD achieves the highest average Pass@1. In the larger 14B$\rightarrow$7B setting,
VCRD achieves an average Pass@1 of 81.75, with 83.5 on HE+ versus 82.3 for the strongest baseline.
Although code generation is more deterministic than mathematical reasoning, VCRD still yields important gains, indicating that prefix-conditioned validity weighting reliably improves student updates even in structured synthesis tasks.

\begin{table*}[t]
\centering

\resizebox{0.7\linewidth}{!}{%
\begin{tabular}{lcccc|c}
\toprule
\textbf{Method}
& \textbf{HEval}
& \textbf{HEval$+$}
& \textbf{MBPP}
& \textbf{MBPP$+$}
& \textbf{AVG} \\
\midrule
\multicolumn{6}{c}{\textbf{Qwen2.5-Coder-7B-Inst ($\mathcal{M}_T$) $\rightarrow$ Qwen2.5-Coder-1.5B ($\mathcal{M}_S$)}} \\
\midrule
$\mathcal{M}_T$
& 91.50 & 86.00 & 82.00 & 70.40 & 82.47 \\

$\mathcal{M}_S$
& 69.50 & 64.00 & 73.00 & 61.40 & 66.97 \\

\midrule
KD
& 68.90 & 61.00 & 72.80 & 61.90 & 66.15 \\

GKD
& 68.30 & 62.80 & 73.30 & 63.20 & 66.90 \\

DistilLLM
& 68.30 & 62.80 & 73.30 & 62.70 & 66.77 \\

ABKD
& 69.50 & 64.00 & 71.20 & 60.60 & 66.32 \\

Distillm-2
& 67.70 & 61.60 & 73.80 & 63.20 & 66.57 \\

\textbf{VCRD}
& 69.50 & 64.60 & 73.80 & 63.00 & \textbf{67.72} \\

\midrule
\multicolumn{6}{c}{\textbf{Qwen2.5-Coder-14B-Inst ($\mathcal{M}_T$) $\rightarrow$ Qwen2.5-Coder-7B-Inst ($\mathcal{M}_S$)}} \\
\midrule
$\mathcal{M}_T$
& 91.50 & 86.60 & 85.40 & 72.80 & 84.07 \\

\midrule
Distillm
& 87.80 & 82.30 & 82.50 & 70.40 & 80.75 \\

Distillm-2
& 88.40 & 82.30 & 83.90 & 70.90 & 81.37 \\

\textbf{VCRD}
& 89.00 & 83.50 & 83.30 & 71.20 & \textbf{81.75} \\

\bottomrule
\vspace{-2em}
\end{tabular}
}
\caption{
Pass@1 results on HumanEval and MBPP benchmarks for two Qwen2.5-Coder distillation settings.
AVG denotes the average Pass@1 across the four tasks.
}
\vspace{0.5em}
\label{tab:heval_mbpp_qwen25_all}
\end{table*}

\vspace{-0.5em}
\subsection{General Instruction-Following}
\vspace{-0.5 em}

\paragraph{Setup.} We evaluate VCRD with Qwen2.5-7B-Instruct $\rightarrow$ Qwen2.5-1.5B and
Qwen2-7B-Instruct $\rightarrow$ Qwen2-1.5B, using UltraChat200k~\citep{ding2023enhancing}.
We report win-rate on AlpacaEval~\citep{li2023alpacaeval}, Evol-Instruct~\citep{xu2024wizardlm}, and
UltraFeedback~\citep{cuiultrafeedback}, using GPT-4o or GPT-4o-mini as an LLM-as-a-Judge~\citep{zheng2023judging} following~\citep{kodistillm}.
\vspace{-1em}
\paragraph{Results.}
Table~\ref{tab:inst_wr} shows that VCRD yields consistent improvements over existing distillation approaches across both configurations. For Qwen2.5-7B$\rightarrow$1.5B, VCRD achieves an average win rate of 66.50, surpassing the strongest baseline DistilLLM (65.47) and improving performance across all three benchmarks.
Even larger gains are observed for Qwen2-7B$\rightarrow$1.5B, where VCRD attains an average win rate of 52.96, outperforming DistillLM-2 (50.61) by +2.35.
These results indicate that VCRD is effective for instruction-following tasks, where 
teacher trajectories exhibit substantial variability. 
\vspace{-0.75 em}
\begin{table*}[h!]
\centering
\vspace{-1 em}
\resizebox{1\linewidth}{!}{%
\begin{tabular}{lcccccccc}
\toprule
\multirow{2}{*}{Method} 
& \multicolumn{4}{c|}{\textbf{Qwen2.5-7B-Inst ($\mathcal{M}_T$) $\rightarrow$ Qwen2.5-1.5B ($\mathcal{M}_S$)}} 
& \multicolumn{4}{c}{\textbf{Qwen2-7B-Inst ($\mathcal{M}_T$) $\rightarrow$ Qwen2-1.5B ($\mathcal{M}_S$)}} \\
\cmidrule(lr){2-5} \cmidrule(lr){6-9}
& AlpacaEval & Evol-Inst & UltraFeed & AVG.
& AlpacaEval & Evol-Inst & UltraFeed & AVG. \\
\midrule

$\mathcal{M}_T$ 
& 91.43 & 69.67 & 77.16 & 79.42
& 88.54 & 59.75 & 64.86 & 71.05 \\

$\mathcal{M}_S$
& 60.09 & 30.73 & 42.94 & 44.59
& 56.11 & 19.49 & 38.20 & 37.93 \\
\midrule

KD 
& 72.14 & 54.59 & 62.60 & 63.11
& 60.81 & 32.45 & 49.25 & 47.50 \\

SeqKD 
& 67.92 & 46.33 & 52.89 & 55.71
& 56.58 & 28.32 & 39.10 & 41.33 \\

DistilLLM 
& 75.11 & 55.38 & 65.92 & 65.47
& 64.13 & 35.66 & 51.26 & 50.35 \\

DistillLM-2
& 73.66 & 56.77 & 65.57 & 65.33
& 62.58 & 37.27 & 51.97 & 50.61 \\
\midrule

\textbf{VCRD}
& 76.02 & 57.80 & 65.67 & \textbf{66.50}
& 62.83 & 44.49 & 51.56 & \textbf{52.96} \\
\bottomrule
\end{tabular}
}
\vspace{-0.75 em}
\caption{
Win Rate (WR\%) on instruction-following datasets for
Qwen2.5-7B-Instruct $\rightarrow$ Qwen2.5-1.5B and 
Qwen2-7B-Instruct $\rightarrow$ Qwen2-1.5B.
AVG is the average WR across all benchmarks.}
\label{tab:inst_wr}
\vspace{-1.5 em}
\end{table*}
\vspace{-0.0 em}
\subsection{Critical Role of Amplification}
\label{sec:amplification}
\vspace{- 0.5em}
A distinctive regime of validity-calibrated supervision occurs when
$w_t>1$, i.e., when the student proposal receives higher local validity
than the sampled teacher proposal under the same prefix. In VCRD,
\emph{amplification} refers specifically to 
assigning greater optimization magnitude to reasoning states,
where the student's locally competitive proposal suggests that the state
lies within its current competence and is more readily exploitable by
distillation.
\begin{figure}
  \centering
  \includegraphics[width=0.44\textwidth]{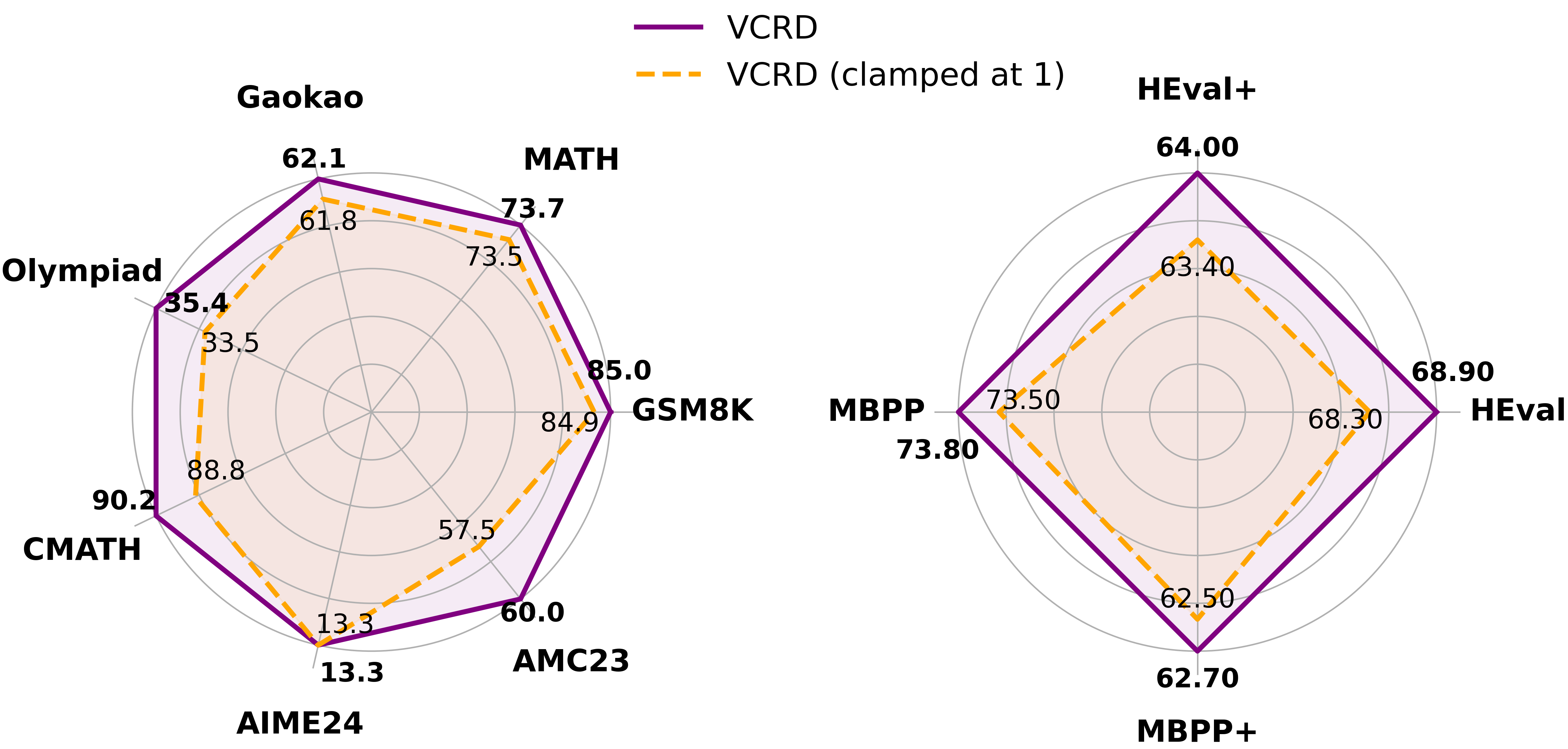}
  \vspace{-0.5 em}
  \caption{Qwen2.5-Math-7B-Inst$\rightarrow$Qwen2.5-Math-1.5B (left) and Qwen2.5-Coder-7B-Inst$\rightarrow$Qwen2.5-Coder-1.5B (right).}
  \label{amplif}
  \vspace{-1.5em}
\end{figure}
To isolate the contribution of this regime, Figure~\ref{amplif}
compares full VCRD with an amplification-disabled variant in which
$
    \tilde{w}_t = \min(w_t,1), $
so that attenuation is preserved for $w_t<1$ while weights above one are
capped. Across both mathematical reasoning and code generation, disabling
amplification consistently reduces performance, with particularly visible
degradation on challenging benchmarks such as AMC23, 
AIME24, and
HumanEval+. These results indicate that allowing high-relative-validity
states to receive 
\vspace{-0.0 em}
greater optimization weight is beneficial beyond
attenuation alone. Appendix~\ref{amplific} further shows that amplified positions account for
a disproportionately large share of the weighted gradient magnitude
across both domains and prefix branches.
\vspace{-0.5em}
\vspace{-0 em}
\subsection{Ablation Study}
\label{ablation}
\vspace{-0.5 em}
\begin{table*}[t]
\centering
\small
\resizebox{\textwidth}{!}{
\begin{tabular}{lcccccccc}
\toprule
\multicolumn{9}{c}{
\textbf{Qwen2.5-Math-7B-Instruct ($M_T$)
$\rightarrow$
Qwen2.5-Math-1.5B ($M_S$)}
} \\
\midrule
\textbf{Method}
& \textbf{GSM8K}
& \textbf{MATH}
& \textbf{Gaokao}
& \textbf{Olympiad}
& \textbf{CMATH}
& \textbf{AIME24}
& \textbf{AMC23}
& \textbf{AVG} \\
\midrule

LV-SKL ($w^T=1$)
& 84.50 & 73.80 & 62.00 & 34.70 & 86.30 & 6.70 & 47.50 & 56.54 \\

LV-SRKL ($w^S=1$)
& 85.70 & 72.90 & 59.00 & 33.90 & 88.00 & 6.70 & 47.50 & 56.24 \\

LV-SKL+LV-SRKL ($w^T=w^S=1$)
& 85.50 & 73.20 & 61.30 & 34.10 & 90.20 & 13.30 & 47.50 & 57.87 \\

\midrule

LV-SKL
& 83.70 & 73.80 & 62.60 & 34.70 & 87.50 & 10.10 & 60.00 & 58.91 \\

LV-SRKL
& 86.60 & 72.80 & 61.00 & 35.10 & 88.70 & 10.00 & 42.50 & 56.67 \\

\midrule

Inverse $r_T/(r_S+\epsilon)$
& 85.40 & 74.10 & 60.00 & 34.10 & 89.00 & 3.30 & 47.50 & 56.20 \\

Mean-normalized $r_S/(r_T+\epsilon)$
& 85.00 & 73.80 & 61.30 & 36.00 & 89.00 & 13.30 & 52.50 & 58.70 \\

\textbf{VCRD}
& 85.00 & 73.70 & 62.10 & 35.40 & 90.20 & 13.30 & 60.00
& \textbf{59.96} \\

\bottomrule
\end{tabular}
}

\vspace{0.6em}

\resizebox{0.83\textwidth}{!}{
\begin{tabular}{lccccc}
\toprule
\multicolumn{6}{c}{
\textbf{Qwen2.5-Coder-7B-Instruct ($M_T$)
$\rightarrow$
Qwen2.5-Coder-1.5B ($M_S$)}
} \\
\midrule
\textbf{Method}
& \textbf{HumanEval}
& \textbf{HumanEval+}
& \textbf{MBPP}
& \textbf{MBPP+}
& \textbf{AVG} \\
\midrule

LV-SKL ($w^T=1$)
& 66.50 & 60.40 & 73.50 & 63.00 & 65.85 \\

LV-SRKL ($w^S=1$)
& 65.50 & 60.40 & 73.30 & 62.20 & 65.45 \\

LV-SKL+LV-SRKL ($w^T=w^S=1$)
& 64.60 & 59.10 & 72.80 & 62.20 & 64.67 \\

\midrule

LV-SKL
& 68.90 & 62.80 & 73.00 & 63.00 & 66.92 \\

LV-SRKL
& 67.10 & 59.80 & 74.30 & 63.20 & 66.10 \\

\midrule

Inverse $r_T/(r_S+\epsilon)$
& 67.10 & 60.40 & 73.50 & 60.30 & 65.33 \\

Mean-normalized $r_S/(r_T+\epsilon)$
& 70.10 & 64.60 & 74.30 & 63.20 & \textbf{68.05} \\

\textbf{VCRD}
& 69.50 & 64.60 & 73.80 & 63.00 & 67.72 \\

\bottomrule
\end{tabular}
}
\vspace{-0.5 em}
\caption{
VCRD across math and code. SKL ($w^T=1$) and SRKL ($w^S=1$) denote single-prefix
distillation without validity modulation; SKL+SRKL
($w^T=w^S=1$) is the constant-weight mixed-prefix baseline.
The inverse and mean-normalized variants provide additional controls on
the direction and global scale of the validity weighting.
}
\label{tab:ablation}
\vspace{-2 em}
\end{table*}

Table~\ref{tab:ablation} examines whether VCRD's gains arise from
validity-conditioned allocation or from simpler properties of the
distillation objective.
\vspace{-0.5 em}
\paragraph{Validity weighting vs.\ mixed-prefix training.}
VCRD outperforms the uniform prefix baseline
on both math and code.
It also improves over the corresponding single-prefix variants,
LV-SKL ($w^T=1$) and LV-SRKL ($w^S=1$).
Thus, the gains cannot be attributed solely to combining teacher- and
student-prefix supervision.
\vspace{-0.75 em}
\paragraph{Controlling for global loss scale.}
To test whether the gains simply arise from increasing the average loss
magnitude, we normalize the weights separately within each prefix branch $
\hat{w}_t^b =
\frac{w_t^b}{\mathbb{E}_{\mathrm{valid}}[w^b]}, b\in\{T,S\}.
$
This preserves the relative variation across positions while giving each
branch mean weight one. The mean-normalized variant outperforms the
uniform mixed-prefix baseline by $0.83$ on math and $3.38$
on code, showing that the gains are not explained by global
loss rescaling.
\paragraph{Does the allocation direction matter?}
We reverse the validity ratio to $r_T/(r_S+\epsilon)$ while retaining the
same position-dependent weighting structure. This inverse allocation
performs worse than the proposed direction, indicating that performance
depends not only on introducing token-dependent weights, but also on how
relative teacher-student validity determines their allocation. Taken together, the uniform, single-prefix, mean-normalized, and inverse
controls indicate that VCRD's gains are associated with the
position-dependent allocation of learning signal according to relative
teacher--student validity, rather than mixed-prefix training or generic
loss rescaling alone.
\vspace{-0.5 em}
\subsection{PRM-free VCRD}
\label{sec:prm_free_main}
\vspace{-0.0 em}
VCRD can also operate when a suitable PRM is unavailable. In
Appendix~\ref{sec:prm_free}, we replace the external validity judge with
a teacher-likelihood proxy for relative local support. This PRM-free
variant outperforms strong distillation baselines on both code generation
and mathematical reasoning.
\vspace{-0.5 em}
\vspace{-0.0 em}
\section{Conclusion}
\vspace{-0.5 em}
This work revisits LLM reasoning distillation from a foundational perspective.
We show that prior trajectory imitation approaches rely on an implicit
monotonicity assumption: that global teacher superiority induces uniformly
reliable local learning signals. This assumption is misaligned with multi-step
reasoning, where intermediate steps are 
often under-specified and the usefulness
of supervision can vary substantially across decision points. Motivated by this
insight, we introduced VCRD, which reframes reasoning distillation as token-level learning-signal
calibration. By comparing the relative local
validity of teacher and student proposals under shared prefixes, VCRD adaptively
modulates the \emph{strength} of the distillation updates while preserving
teacher-anchored optimization geometry. 
It
attenuates learning signal at states with a larger local teacher-student
capability gap and amplifies it at student-competitive states that appear
more compatible with the student's current competence, without changing
the direction of distributional supervision. 
Empirically, VCRD yields consistent improvements across math,
code, and instruction-following tasks, including settings with
strong teachers and diverse solution spaces.
Ablations further support the role of
relative-validity weighting,
consistent with the theoretical motivation.
Our results
show that effective reasoning distillation requires calibrating \emph{how
strongly} models learn from intermediate steps, rather than enforcing uniform
imitation of a single reasoning trajectory. 
VCRD performs similarly with different PRMs, supporting robustness to PRM choice. Deliberately corrupting PRM-derived teacher-student validity ordering reduces the gains, highlighting the value of informative relative-validity estimates.

\section{Limitations}
Although PRMs provide a natural way to compare teacher and student
continuations under a shared prefix, their availability and calibration
remain limited across model families and reasoning domains.
To reduce this dependence, we introduced a PRM-free variant that uses
teacher likelihood as a proxy for local validity, retaining relative
local calibration without an external reward model.
Preliminary results show improvements over strong distillation
baselines, suggesting that VCRD's allocation principle extends beyond
a specific PRM. However, teacher likelihood remains an imperfect
surrogate for local validity and may not capture all aspects of
step-level reasoning quality. Developing general, well-calibrated
validity estimators across domains therefore remains an important
direction for future work.

\bibliography{custom}
\newpage
\clearpage
\appendix

\section{Qualitative Examples}
\label{app:student_better_examples}
In Table~\ref{tab:student_better_numeric_examples}, we present qualitative examples in which the student continuation receives a higher local validity reward than the teacher continuation under the same prefix. These examples illustrate that a globally stronger teacher can still propose a less locally valid continuation than the student.
\begin{table*}[!t]
\centering
\scriptsize
\resizebox{\linewidth}{!}{%
\begin{tabular}{p{0.28\linewidth} p{0.16\linewidth} p{0.16\linewidth} ccc p{0.18\linewidth}}
\toprule
\textbf{Shared prefix}
& \textbf{Teacher continuation}
& \textbf{Student continuation}
& $\boldsymbol{r_T}$
& $\boldsymbol{r_S}$
& \textbf{Ratio}
& \textbf{Better Local continuation} \\
\midrule

The cost of printing is $7$ copies $\times$ $25$ pages $\times$ \$0.10/page $= \$17$
& $550$
& $.50$
& 0.085
& 0.589
& 6.97
& \textbf{Student}. The correct amount is \$17.50; the student completes the decimal. \\
\midrule

Tim buys 3 goats for \$400 each. Tim also buys twice as many llamas as goats, so he buys
& $2 \times 2 = 6$ llamas
& $3 \times 2 = 6$ llamas
& 0.152
& 0.847
& 5.57
& \textbf{Student}. The multiplier should be the number of goats, 3, so $3 \times 2 = 6$. \\



\midrule

Cost per liter is \$3 and the pool volume is $25 \times 60x$. Thus, $3 \times 25 \times 60x = 90000$. Simplify: $3 \times$
& $15 \times 60x$
& $25 \times 60x$
& 0.575
& 0.801
& 1.39
& \textbf{Student}. The coefficient 25 should be preserved from $3 \times 25 \times 60x$. \\

\midrule

Find $1^{234} + 4^6 \div 4^4$. First, we simplify
& $1^6 \div 4^4$
& $4^6 \div 4^4$
& 0.601
& 0.877
& 1.46
& \textbf{Student}. The expression being simplified is $4^6 \div 4^4$, not $1^6 \div 4^4$. \\

\midrule

James drives at 30 mph for half an hour. Then he drives at twice the speed for twice as long, which means he drives at a speed of
& $3 \times 30 = 6\cdots$
& $2 \times 30 = 6\cdots$
& 0.360
& 0.827
& 2.30
& \textbf{Student}. ``Twice the speed'' requires multiplying by 2, giving $2 \times 30 = 60$ mph. \\

\bottomrule
\end{tabular}
}
\caption{
Qualitative examples where the student Qwen2.5-Math-1.5B continuation receives higher local validity than the teacher Qwen2.5-Math-7B-Instruct continuation under the same prefix.
$r_T$ and $r_S$ denote the PRM (Skywork-o1-OpenPRM-Qwen-2.5-1.5B) rewards for the teacher and student \textbf{first token} continuations, respectively, and the ratio is $r_S/r_T$.
The ``Better Local Continuation'' column indicates which continuation is locally correct or a valid better alternative under the shared prefix between teacher and student models.
}
\label{tab:student_better_numeric_examples}
\end{table*}

\section{Related Work} \label{rl}

\textbf{Reasoning.}  
Recent advances in large language model reasoning have been driven by explicit
multi-step supervision techniques, including chain-of-thought prompting
\citep{wei2022chain}, structured reasoning templates
\citep{chenglin2024mixed,zhu2024distilling}, thought-expansion methods such as
Tree-of-Thought and related variants \citep{yao2023tree}, and other forms of
guided step decomposition \citep{shridhar2023distilling}. These approaches
substantially improve the reasoning performance of large models but typically
depend on scale and heavy computation, limiting their applicability in
resource-constrained settings. This has motivated growing interest in enhancing the reasoning abilities of
smaller models through two main directions: (i) richer training-time supervision
schemes, such as self-correction \citep{madaan2023self} and preference-based
learning \citep{yangsupercorrect}, and (ii) knowledge distillation
\citep{schmidhuber1992learning,hinton2015distilling}, where a strong teacher
transfers its reasoning behavior to a compact student via teacher-generated
traces \citep{li2022explanations,magister2023teaching,liu2024deepseek}.

\textbf{Reasoning Distillation.}  
Reasoning distillation has emerged as an effective approach for transferring
chain-of-thought capabilities from larger to smaller language models
\citep{mukherjee2023orca,hsieh2023distilling,mitra2024orca,fu2023specializing},
allowing students to inherit strong reasoning behavior without costly training
from scratch. Early work framed this as rationale imitation
\citep{ho2023large,liu2024deepseek}, training students to reproduce teacher
explanations using cross-entropy or sequence-level KD
\citep{hinton2015distilling,kim-rush-2016-sequence}. Subsequent work extends
this paradigm with trajectory selection and filtering
\citep{zelikman2022star,lightman2023let,wangself}, keeping only high-quality or
self-consistent traces, and with alternative objectives such as on-policy
distillation using student rollouts \citep{agarwal2024policy,ko2024distillm} or
contrastive losses that balance teacher-guided and student-generated supervision
\citep{kodistillm}. Despite this progress, existing techniques largely treat reasoning trajectories
as monolithic sequences: implicitly assuming that all teacher steps are equally informative or
equally authoritative. This creates two limitations: (i) \emph{local
under-specification}, many reasoning prefixes admit multiple valid
continuations, yet sequence-level KD enforces a single teacher choice; and (ii)
\emph{uniform optimization}, the student receives identical learning signal on
strong teacher steps, weak teacher steps, and student-favored steps. In
contrast, our method adopts a token-level, prefix-conditioned view of reasoning
distillation.
\section{Theoretical Results and Proofs}
\label{appendA}

This appendix provides formal statements and proofs supporting the theoretical
claims in Section~\ref{sec:theory}. Throughout, we consider a fixed prefix
(context) \(c\) and a finite vocabulary \(\mathcal{V}\). The teacher policy is
denoted by \(\pi(\cdot\mid c)\), and \(r(c,a)\) denotes a bounded local validity
signal.

\subsection{Teacher-Anchored KL Trust-Region Improvement (detailed proof)}
\label{subse1}
\begin{theorem}[Teacher-anchored reverse-KL trust-region solution]
\label{thm:exp_tilt_fixed}
Fix a prefix $c$ and a teacher policy $\pi(\cdot\mid c)$
over a finite set $\mathcal{V}$, with $\pi(a\mid c)>0$
for all $a\in\mathcal{V}$. Let $r(c,a)$ be bounded
and nonconstant over $\mathcal{V}$. Define:
\[
\begin{aligned}
A_{\max}
&= \arg\max_{a\in\mathcal{V}} r(c,a), \\
\pi(A_{\max}\mid c)
&= \sum_{a\in A_{\max}}\pi(a\mid c).
\end{aligned}
\]
For a radius satisfying:
\[
0<\delta<-\log\pi(A_{\max}\mid c),
\]
consider:
\begin{equation}
\label{eq:append_trust_region_fixed}
\begin{aligned}
\max_{\tilde{\pi}(\cdot\mid c)}\quad&
\sum_{a\in\mathcal{V}}
\tilde{\pi}(a\mid c)\,r(c,a) \\
\text{s.t.}\quad&
\mathrm{KL}\!\left(
\tilde{\pi}(\cdot\mid c)
\,\|\,\pi(\cdot\mid c)
\right)\le\delta, \\
&\sum_a\tilde{\pi}(a\mid c)=1, \\
&\tilde{\pi}(a\mid c)>0
\quad \forall a\in\mathcal{V}.
\end{aligned}
\end{equation}
The unique optimal solution is:
\begin{equation}
\label{eq:append_exp_tilt_fixed}
\tilde{\pi}^{\star}(a\mid c)
=
\frac{\pi(a\mid c)e^{\eta r(c,a)}}
{\sum_{a'\in\mathcal{V}}
 \pi(a'\mid c)e^{\eta r(c,a')}},
\end{equation}
where the unique finite $\eta>0$ satisfies
$\mathrm{KL}(\tilde{\pi}^{\star}\|\pi)=\delta$.
\end{theorem}

\begin{proof}
Fix $c$ and abbreviate $\pi(a)=\pi(a\mid c)$
and $r(a)=r(c,a)$. Define:
\[
\begin{aligned}
Z(\eta)
&= \sum_a\pi(a)e^{\eta r(a)}, \\
\pi_\eta(a)
&= \frac{\pi(a)e^{\eta r(a)}}{Z(\eta)}.
\end{aligned}
\]
Every $\pi_\eta$ has full support for finite $\eta$.
Its divergence from the teacher is:
\[
\begin{aligned}
K(\eta)
&= \mathrm{KL}(\pi_\eta\|\pi) \\
&= \eta\,\mathbb{E}_{\pi_\eta}[r]
   -\log Z(\eta).
\end{aligned}
\]
For $\eta>0$, differentiation gives:
\[
K'(\eta)
= \eta\,\operatorname{Var}_{\pi_\eta}[r]>0,
\]
because the rewards are nonconstant and
$\pi_\eta$ has full support. Moreover,
$K(0)=0$ and
\[
\lim_{\eta\to\infty}K(\eta)
= -\log\pi(A_{\max}\mid c).
\]
Hence the assumed radius admits a unique finite
$\eta>0$ such that $K(\eta)=\delta$.

For any feasible policy $\tilde{\pi}$,
\[
\begin{aligned}
\mathrm{KL}(\tilde{\pi}\|\pi_\eta)
&= \mathrm{KL}(\tilde{\pi}\|\pi) \\
&\quad -\eta\,\mathbb{E}_{\tilde{\pi}}[r]
       +\log Z(\eta).
\end{aligned}
\]
Using
$\delta=\eta\,\mathbb{E}_{\pi_\eta}[r]-\log Z(\eta)$
and rearranging yields:
\[
\begin{aligned}
&\eta\left(
\mathbb{E}_{\tilde{\pi}}[r]
-\mathbb{E}_{\pi_\eta}[r]
\right) \\
&\quad =
\mathrm{KL}(\tilde{\pi}\|\pi)-\delta \\
&\qquad
-\mathrm{KL}(\tilde{\pi}\|\pi_\eta)
\le 0.
\end{aligned}
\]
Thus $\pi_\eta$ maximizes expected reward.
Equality requires
$\mathrm{KL}(\tilde{\pi}\|\pi_\eta)=0$,
so the optimum is unique. Therefore
$\tilde{\pi}^{\star}=\pi_\eta$, giving
Eq.~\ref{eq:append_exp_tilt_fixed},
and the KL constraint is active.
\end{proof}

\subsection{Effect of Detached Validity Weighting on the Distillation Gradient}
\label{app:detached_weighting}

We now characterize the optimization effect of the practical
validity-weighting rule used by VCRD. This result is distinct from the
full-action trust-region solution in Appendix~\ref{subse1}:
the trust-region analysis motivates comparing actions by their relative
validity under the same prefix, whereas the ratio
$r(c_t,a_t^S)/(r(c_t,a_t^T)+\epsilon)$ is a practical two-proposal
allocation rule.

Fix a prefix $c_t$ and the corresponding teacher and student proposals.
Let $D_t(\theta)$ denote either of the distributional distillation terms
used by VCRD (SKL or SRKL), and let $w(c_t) > 0$ denote the associated
validity weight. During optimization, $w(c_t)$ is treated as a detached
scalar.

\paragraph{Proposition C.2 (Detached weighting rescales the local gradient).}
For the weighted local objective
\begin{equation}
    \mathcal{L}_t(\theta)
    =
    w(c_t) D_t(\theta),
\end{equation}
with $\operatorname{stopgrad}(w(c_t))$, its gradient with respect to the
student parameters is
\begin{equation}
    \nabla_\theta \mathcal{L}_t(\theta)
    =
    w(c_t)\nabla_\theta D_t(\theta).
\end{equation}

\paragraph{Proof.}
Because $w(c_t)$ is detached from the computation graph, it is constant
with respect to $\theta$. Therefore,
\begin{equation}
\begin{aligned}
\nabla_\theta
\left[
    w(c_t)D_t(\theta)
\right]
&=
\nabla_\theta w(c_t)\,D_t(\theta)
\\
&\quad+
w(c_t)\nabla_\theta D_t(\theta)
\\
&=
w(c_t)\nabla_\theta D_t(\theta).
\end{aligned}
\end{equation}
since $\nabla_\theta w(c_t)=0$ under stop-gradient.
\hfill$\square$

Consequently, the validity weight changes the magnitude with which the
decision point contributes to optimization, while the underlying
distillation divergence determines the vocabulary-level update
direction. Thus, $w(c_t)<1$ attenuates the local learning signal and
$w(c_t)>1$ amplifies it. Importantly, amplification does not imply that
the probability of the sampled student action necessarily increases;
that behavior depends on the complete teacher and student distributions
through the SKL/SRKL gradient.

This proposition characterizes the optimization effect of the detached
weight but does not uniquely derive the particular ratio
\begin{equation}
    w(c_t)
    =
    \frac{r(c_t,a_t^S)}
         {r(c_t,a_t^T)+\epsilon}.
\end{equation}
As discussed in Section~\ref{sec:theory}, this ratio is motivated by the
within-prefix relative-validity structure of the full-action
trust-region solution and is empirically evaluated against alternative
allocation rules in Subsection~\ref{ablation}.

\section{Algorithm}
In this section, we provide our algorithm~\ref{alg:vcrd}.
\label{appendB}

\captionof{algorithm}{Validity-Calibrated Reasoning Distillation (VCRD)}
\label{alg:vcrd}
\begin{algorithmic}[1]
\STATE \textbf{Input:} Dataset $\mathcal{D}=\{x_i\}$; teacher LM $p$; student LM $q_\theta$;
auxiliary judge $J$; loss weights $(\lambda_T,\lambda_S)$; smoothing $\varepsilon>0$;
rollout horizon $T$.
\STATE \textbf{Output:} Updated student parameters $\theta$.

\STATE \textbf{Local validity:} For any prefix $c$ and token $a$, define
$r(c,a) = J(c,a) \in [0,1]$.

\STATE \textit{// Generate teacher and student trajectories once before training}
\STATE Let $\theta_0 \gets \theta$.
\FOR{\textbf{each} input $x\in\mathcal{D}$}
\STATE $y^T(x)\sim p(\cdot\mid x)$.
\STATE $y^S(x)\sim q_{\theta_0}(\cdot\mid x)$.
\ENDFOR

\FOR{\textbf{each} training step}
    \STATE Sample minibatch $\mathcal{B}\subset\mathcal{D}$.
    \FOR{\textbf{each} input $x\in\mathcal{B}$}

        \STATE \textit{// Retrieve fixed teacher and student trajectories}
        \STATE $y^T \gets y^T(x)$ and $y^S \gets y^S(x)$.

        \STATE \textit{// Prefix definitions}
        \STATE $c_{t-1}^T = (x, y_{<t}^T)$ and
        $c_{t-1}^S = (x, y_{<t}^S)$ for $t=1,\dots,T$.
\STATE \textit{// LV--SKL weights (teacher prefix)}
\FOR{$t=1$ to $T$}
    \STATE $a_t^T \gets \arg\max_a p(a\mid c^T_{t-1})$
and $a_t^S \gets \arg\max_a q_\theta(a\mid c^T_{t-1})$.
    \STATE $w_t^T \gets
    \dfrac{r(c^T_{t-1},a_t^S)}
          {r(c^T_{t-1},a_t^T)+\varepsilon}$.
\ENDFOR
    
\STATE \textit{// LV--SRKL weights (student prefix)}
\FOR{$t=1$ to $T$}

\STATE $a_t^T \gets \arg\max_a p(a\mid c^S_{t-1})$
and $a_t^S \gets \arg\max_a q_\theta(a\mid c^S_{t-1})$.
    \STATE $w_t^S \gets
    \dfrac{r(c^S_{t-1},a_t^S)}
          {r(c^S_{t-1},a_t^T)+\varepsilon}$.
\ENDFOR

        \STATE \textit{// Validity-weighted KL losses}
        \STATE $\displaystyle 
        \mathcal{L}_{\mathrm{LV\text{-}SKL}}(x)
        =
        \sum_{t=1}^{T} 
        w^T_t \,
        D^{(\alpha)}_{\mathrm{SKL}}\!\left(
            p(\cdot\mid c^T_{t-1})
            \,\big\|\,
            q_\theta(\cdot\mid c^T_{t-1})
        \right)$
        
        \STATE $\displaystyle
        \mathcal{L}_{\mathrm{LV\text{-}SRKL}}(x)
        =
        \sum_{t=1}^{T} 
        w^S_t \,
        D^{(\alpha)}_{\mathrm{SRKL}}\!\left(
            p(\cdot\mid c^S_{t-1})
            \,\big\|\,
            q_\theta(\cdot\mid c^S_{t-1})
        \right)$
    \ENDFOR

    \STATE \textit{// Batch aggregation and update}
    \STATE $\displaystyle 
    \mathcal{L} = \frac{1}{|\mathcal{B}|}
    \sum_{x\in\mathcal{B}}
    \Big[
        \lambda_T \, \mathcal{L}_{\mathrm{LV\text{-}SKL}}(x)
        + 
        \lambda_S \, \mathcal{L}_{\mathrm{LV\text{-}SRKL}}(x)
    \Big]$
    
    \STATE Update $\theta \leftarrow \theta - \eta\,\nabla_\theta \mathcal{L}$
\ENDFOR
\end{algorithmic}

\section{Detailed Experimental Setup} 
\label{expsetup}
\begin{table*}[!t]
\centering
\begin{tabular}{lc}
\toprule
\textbf{Hyperparameter} & \textbf{Value} \\
\midrule

Finetuning method & LoRA (rank = 16, $\alpha$ = 128, dropout = 0.05) \\

LoRA target modules & all self-attention and MLP layers \\

Learning rate & $5 \times 10^{-5}$ \\

Max sequence length & 1024 \\

Max prompt length & 512 \\

Effective batch size & 128 \\

LR scheduler & cosine \\

Warmup ratio & 0.1 \\

Validity smoothing $\varepsilon$ & $1 \times 10^{-8}$ \\


\# Epochs (Math Reasoning) & 2 epochs \\

\# Epochs (Code Generation) & 1 epoch (600/400 iterations) \\

\# Epochs (Instruction Following) & 3 epochs \\
SKL/SRKL mixing coefficient $\alpha$ & 0.1 \\

\bottomrule
\end{tabular}
\caption{Hyperparameter values used in VCRD (distillation stage) experiments across all task families.}
\label{hyperparams}
\end{table*}

\subsection{Dataset Description}
\begin{table*}[t]
\centering
\small
\begin{tabular}{lcccc}
\toprule
& \multicolumn{2}{c}{\textbf{Teacher-prefix SKL}}
& \multicolumn{2}{c}{\textbf{Student-prefix SRKL}} \\
\cmidrule(lr){2-3}\cmidrule(lr){4-5}

\textbf{Teacher $\rightarrow$ Student}
& \textbf{Positions (\%)}
& \textbf{Gradient (\%)}
& \textbf{Positions (\%)}
& \textbf{Gradient (\%)} \\
\midrule

Qwen2.5-Math-7B $\rightarrow$ Qwen2.5-Math-1.5B
& 26.21 & 43.43
& 52.97 & 60.43 \\

Qwen2.5-Coder-7B $\rightarrow$ Qwen2.5-Coder-1.5B
& 19.72 & 40.43
& 28.12 & 46.08 \\

\bottomrule
\end{tabular}
\caption{
Optimization allocation over amplified positions ($w_t \geq 1$).
``Positions'' denotes their percentage among valid loss positions, and
``Gradient'' their share of the total weighted per-token logit-gradient norm.
}
\label{tab:gradient_allocation}
\end{table*}
\textbf{MetaMathQA} (mathematical reasoning;~\citet{yumetamath}\textsuperscript{1}): MetaMathQA is a large-scale dataset introduced to strengthen mathematical reasoning in language models. It is constructed through question bootstrapping, where each problem is rewritten from multiple perspectives, such as forward reasoning, backward reasoning, and alternative rephrasings, to expose models to diverse solution paths and intermediate reasoning structures.

\textbf{GSM8K} (mathematical reasoning;~\citet{cobbe2021training}\textsuperscript{2}): 
GSM8K consists of 8.5K high-quality grade-school math word problems that require multi-step numerical and logical reasoning. The dataset emphasizes clarity, linguistic diversity, and systematic solution steps, making it a standard benchmark for evaluating fundamental arithmetic reasoning in LLMs.

\textbf{MATH} (mathematical reasoning;~\citet{hendrycks2measuring}\textsuperscript{3}): 
The MATH dataset contains thousands of competition-style mathematical questions spanning algebra, geometry, combinatorics, number theory, probability, and more. It includes detailed step-by-step solutions generated from a procedural codebase, enabling rigorous evaluation of a model’s ability to perform symbolic and multi-step mathematical reasoning at a variety of difficulty levels.

\textbf{SVAMP} (mathematical reasoning;~\citet{patel2021nlp}\textsuperscript{4}): 
SVAMP is a variant of GSM8K designed to mitigate spurious pattern exploitation in arithmetic word problems. It introduces controlled perturbations, such as swapping quantities or modifying irrelevant details, to evaluate whether models rely on genuine mathematical reasoning rather than superficial cues.

\textbf{MinervaMath} (mathematical reasoning;~\citet{lewkowycz2022solving}\textsuperscript{5}): 
MinervaMath is a challenging benchmark derived from the Minerva project, focusing specifically on mathematical problem solving. It includes competition-style questions across algebra, calculus, geometry, number theory, and probability. Problems require multi-step symbolic manipulation and precise derivations, making MinervaMath a rigorous test of a model’s ability to handle advanced, long-chain mathematical reasoning.

\textbf{Gaokao2023-EN} (mathematical reasoning;~\citet{zhang2023evaluating}\textsuperscript{6}): 
Gaokao2023-EN contains English translations of math questions from the 2023 Chinese National College Entrance Examination (Gaokao). The dataset features concise, exam-style problems across algebra, geometry, trigonometry, and applied math. Its formulation emphasizes careful reading, multi-step reasoning, and robustness to linguistically minimal prompts, providing a strong evaluation of structured mathematical reasoning under realistic exam conditions.

\textbf{OlympiadBench} (mathematical reasoning;~\citet{he2024olympiadbench}\textsuperscript{7}): 
OlympiadBench aggregates Olympiad-style mathematical problems requiring multi-hop symbolic reasoning, pattern discovery, and structured derivations. Questions are high difficulty and typically require creative reasoning, making this benchmark sensitive to errors in intermediate steps.  

\textbf{AMC23} (mathematical reasoning;~\citet{aimo2024amc}\textsuperscript{8}): 
AMC23 contains problems from the 2023 American Mathematics Competition, focusing on algebra, geometry, combinatorics, and number theory at the mid-competition level. The dataset evaluates a model’s ability to navigate moderately challenging problems that require structured reasoning rather than memorized patterns.

\textbf{AIME24} (mathematical reasoning;~\citet{maa2025aime}\textsuperscript{9}): 
AIME24 consists of problems from the 2024 American Invitational Mathematics Examination. These questions demand multi-step derivations, precise algebraic manipulation, and careful numerical reasoning, providing a sensitive test of a model’s ability to avoid compounding local reasoning errors.

\textbf{SAT-Math} (mathematical reasoning;~\citet{zhong2024agieval}\textsuperscript{10}): 
SAT-Math evaluates models on algebra, arithmetic reasoning, function interpretation, and geometry tasks from the SAT exam. While less challenging than competition benchmarks, the dataset tests robustness under shorter, mixed-format reasoning questions.

\textbf{CMATH} (mathematical reasoning;~\citet{wei2023cmath}):
CMATH is a curated collection covering a broad set of competition-math problem types, with carefully structured reasoning paths and high-quality solutions. It includes tasks requiring symbolic manipulation, equation solving, and multi-step deductive reasoning.

\textbf{WizardCoder} (code generation;~\citet{luowizardcoder}\textsuperscript{11}):
WizardCoder is constructed using the Evol-Instruct procedure, which automatically expands and refines existing code-instruction corpora. The process begins from CodeAlpaca (20K instructions) and iteratively applies instruction evolution techniques, adding constraints, increasing reasoning depth, introducing distractor code, and modifying specifications, to produce more challenging programming tasks. The resulting dataset contains roughly 78K evolved problems and serves as a large-scale code-instruction corpus used to finetune StarCoder and related models.

\textbf{HumanEval.} (code generation;~\citet{chen2021evaluating}\textsuperscript{12}): 
HumanEval is a benchmark of 164 manually written Python programming problems, each comprising a function signature, natural-language description, and a set of unit tests. The tasks were explicitly created to avoid overlap with pretraining corpora, making HumanEval a standard benchmark for evaluating functional correctness of program synthesis.

\paragraph{HumanEval+.}
HumanEval+ extends HumanEval with a substantially enlarged set of
test cases, providing more rigorous evaluation of functional correctness
and exposing solutions that pass the original tests but fail on additional
edge cases.

\textbf{MBPP.} (code generation;~\citet{austin2021program}\textsuperscript{13}): 
MBPP contains approximately 1,000 crowdsourced Python programming tasks aimed at entry-level programmers. Each problem includes a description and test cases covering basic programming constructs such as loops, list manipulation, strings, and simple algorithms. MBPP is widely used to measure fundamental code-generation abilities and correctness.

\paragraph{MBPP+.}
MBPP+ similarly extends MBPP with a substantially enlarged set of test
cases, enabling more rigorous evaluation of the functional correctness
and robustness of generated programs.

\footnotetext[1]{\url{https://huggingface.co/datasets/meta-math/MetaMathQA}}
\footnotetext[2]{\url{https://huggingface.co/datasets/openai/gsm8k}}
\footnotetext[3]{\url{https://huggingface.co/datasets/deepmind/math}}
\footnotetext[4]{\url{https://huggingface.co/datasets/ChilleD/SVAMP}}
\footnotetext[5]{\url{https://huggingface.co/datasets/math-ai/minervamath}}
\footnotetext[6]{\url{https://huggingface.co/datasets/MARIO-Math-Reasoning/Gaokao2023-Math-En}}
\footnotetext[7]{\url{https://huggingface.co/datasets/Hothan/OlympiadBench}}
\footnotetext[8]{\url{https://huggingface.co/datasets/math-ai/amc23}}
\footnotetext[9]{\url{https://huggingface.co/datasets/math-ai/aime24}}
\footnotetext[10]{\url{https://huggingface.co/datasets/ndavidson/sat-math-chain-of-thought}}
\footnotetext[11]{\url{https://huggingface.co/datasets/nickrosh/Evol-Instruct-Code-80k-v1}}
\footnotetext[12]{\url{https://huggingface.co/datasets/openai/openai_humaneval}}
\footnotetext[13]{\url{https://huggingface.co/datasets/google-research-datasets/mbpp}}

\subsection{Training Details}
\label{appendix:training_details}

In this subsection, we describe the hyperparameters and implementation settings used for training VCRD. 
Table~\ref{hyperparams} summarizes all configuration values. 
All models are trained using LoRA-based adaptation~\citep{hu2022lora} for parameter-efficient finetuning, and we adopt a unified optimization setup across mathematical reasoning, code generation, and instruction-following tasks. 
We use the maximum batch size that fits on our 4~NVIDIA A100~(80GB) GPUs, combined with gradient accumulation, to attain an effective batch size of~128. Where applicable, student models are first initialized via supervised
finetuning on task-specific datasets with ground-truth responses, after
which validity-calibrated distillation is applied.
Following the setup of~\citet{kodistillm}, we do not use any language-modeling loss on pretraining corpora. For all experiments, we used FlashAttention and bf16 precision.

For mathematical reasoning, we set the distillation weights to $(\lambda_T,\lambda_S)=(1,1)$ for the Qwen2.5-Math experiments and $(2,1.5)$ for the Qwen2-Math experiments. 
Students are first initialized by supervised fine‑tuning on the full MetaMathQA~\citep{yumetamath} dataset for one epoch  with learning rate $5\times 10^{-6}$, cosine learning-rate scheduling, maximum sequence length 2048, and warmup ratio~0.1. Then, for distillation, we construct a training set by randomly sampling  50k examples from MetaMathQA and collecting both teacher and student responses for each prompt. Distillation performed for two epochs. 

For code generation, we evaluate VCRD on two teacher--student configurations:
Qwen2.5-Coder-7B-Instruct$\rightarrow$Qwen2.5-Coder-1.5B and
Qwen2.5-Coder-14B-Instruct$\rightarrow$Qwen2.5-Coder-7B-Instruct. 
In the 7B$\rightarrow$1.5B setting, we first finetune the 1.5B student on the 
ground-truth code dataset for 5 epochs using FlashAttention, bf16 precision, 
a learning rate of $5\!\times\!10^{-6}$, a maximum sequence length of 2048, 
and a warmup ratio of $0.1$. We then perform VCRD distillation for one epoch 
(600 training iterations). 
For the larger 14B$\rightarrow$7B configuration, we initialize the 7B student 
directly from the Qwen2.5-Coder-7B-Instruct checkpoint without any supervised 
finetuning, and distill for a single epoch of 400 iterations. 
For both configurations, we set the distillation weights to 
$(\lambda_T,\lambda_S) = (2,1)$. Distillation prompts are from the WizardCoder dataset~\citep{luowizardcoder}, which is constructed using Evol-Instruct method~\citep{xu2024wizardlm} code instruction datasets. For distillation, we construct a training set by collecting both teacher and student responses for each prompt from WizardCoder dataset.

This setup, spanning two distinct model scales, enables us to 
assess whether VCRD's prefix-conditioned validity calibration 
provides consistent improvements in execution-level robustness 
and functional correctness across code-generation regimes.

For instruction-following, following~\citet{kodistillm}, 
we first supervise-finetune the student for one epoch
on UltraChat200k using the provided responses.
During this stage, we set the maximum
sequence length to 2048, the learning rate to $5\times10^{-5}$, and the warmup
ratio to $0.1$. After finetuning, we perform VCRD distillation for three epochs
over the sampled UltraChat prompts. The distillation weights are set to
$(\lambda_T,\lambda_S)=(2,2)$ for the Qwen2.5 teacher-student configuration and
$(2,1)$ for the Qwen2 configuration. For distillation, we construct a 
training set by randomly sampling 50k prompts from the UltraChat200k 
dataset~\citep{ding2023enhancing} and collecting both teacher and student 
responses for each prompt. We evaluate VCRD under two teacher-student 
configurations: Qwen2.5-7B-Instruct$\rightarrow$Qwen2.5-1.5B and 
Qwen2-7B-Instruct$\rightarrow$Qwen2-1.5B. Distillation is run for three epochs 
over the sampled UltraChat prompts.
For evaluation, we report win-rate performance on three common 
instruction-following benchmarks: AlpacaEval~\citep{li2023alpacaeval}, 
Evol-Instruct~\citep{xu2024wizardlm}, and UltraFeedback~\citep{cuiultrafeedback}. 
Following prior work~\citep{kodistillm}, we use GPT-4o or GPT-4o-mini as 
LLM-as-a-Judge~\citep{zheng2023judging} to provide consistent preference-based 
scoring across models. This benchmark suite spans conversational, multi-step, 
and preference-heavy instructions, enabling a broad and challenging evaluation 
of VCRD's ability to provide stable supervisory signals beyond purely 
reasoning-focused tasks.

\subsection{Evaluation}

\label{sec:eval_appendix}
We follow a similar evaluation protocol to~\citep{kodistillm}. Below we describe the specific
settings used for each task family.

\textbf{Math Reasoning.}
For evaluating mathematical reasoning performance,  we use a single A100~80GB GPU and we follow the official
Qwen2.5-Math evaluation protocol\footnote{\url{https://github.com/QwenLM/Qwen2.5-Math}}, using the
\texttt{qwen25-math-cot} prompt format. All evaluations are performed with
greedy decoding (temperature~0, top-$p$=1) and a single sample per query under
the vLLM backend. 

\textbf{Code Generation.}
For code evaluations, we again use a single A100~80GB GPU and employ
greedy decoding with a maximum generation length of~1024. Code generation performance is measured using the EvalPlus framework
\citep{liu2023your}, which executes predicted solutions to verify
correctness. For code generation, we use the HumanEval, HumanEval+, MBPP, and
MBPP+ evaluation suites, all executed with their official test harnesses to
ensure consistency and prevent overfitting to reference implementations.

\textbf{Instruction following.}
For instruction-following evaluation, we generate model responses
using a single NVIDIA A100~80GB GPU with temperature~0.8, top-$p$=0.95, and a
maximum generation length of~512 tokens. Each comparison is performed using the
pairwise system prompt described in Appendix~\ref{sec:if_judge_prompt}, which
presents the judge model with the user question and two candidate responses and
asks for a preference decision. To mitigate position bias, we randomly swap the
order of the two responses and average win rates across both permutations. For AlpacaEval, we use the officially released \texttt{text-davinci-003}
reference responses. For Evol-Instruct and UltraFeedback, we compare generated
responses against \texttt{gpt-3.5-turbo} outputs that were produced internally,
following the same protocol used in prior benchmark releases. All evaluations
use GPT-4o or GPT-4o-mini as the judge to ensure consistent preference-based
scoring across models.

\subsubsection{Instruction-Following Judge Prompt}
\label{sec:if_judge_prompt}

\begin{tcolorbox} [colback=white,colframe=black!15, left=1mm, right=1mm, top=1mm, bottom=1mm]
\textbf{[System]} \\
Act as an impartial judge and evaluate which of two assistant responses better
answers the user question shown below. Your judgment should consider correctness,
instruction compliance, coherence, clarity, and overall helpfulness. Ignore
superficial stylistic differences unless they affect content quality.

Read both responses carefully and provide a brief explanation of your choice.
To avoid position bias, do not let the order of presentation influence your
decision. After your explanation, output your verdict strictly in one of the
following formats: \texttt{[[A]]}, \texttt{[[B]]}, or \texttt{[[C]]} (tie).

\vspace{0.5em}
\textbf{[User Question]} \\
\texttt{\{question\}}

\vspace{0.5em}
\textbf{[Response A]} \\
\texttt{\{answer\_A\}}

\vspace{0.5em}
\textbf{[Response B]} \\
\texttt{\{answer\_B\}}
\end{tcolorbox}

\subsection{Role of Amplification} \label{amplific}
\begin{table} 

\centering
\resizebox{1\columnwidth}{!}{%
\begin{tabular}{lcc|c}
\toprule
\textbf{Method}
& \textbf{HEval}
& \textbf{HEval$+$}
& \textbf{AVG} \\
\midrule

$\mathcal{M}_T$
& 77.40 & 70.70 & 74.05 \\

$\mathcal{M}_S$
& 35.40 & 29.30 & 32.35 \\
\midrule

KD
& 42.10 & 37.20 & 39.65 \\

SeqKD
& 41.50 & 36.00 & 38.75 \\

GKD
& 40.09 & 36.00 & 38.05 \\

DistilLLM
& 42.10 & 38.40 & 40.25 \\

ABKD
& 42.10 & 37.80 & 39.95 \\

Distillm-2
& 42.70 & 38.40 & 40.55 \\
\midrule

\textbf{VCRD-Prob}
& \textbf{44.50} & \textbf{40.90} & \textbf{42.70} \\
\bottomrule
\end{tabular}
}
\caption{
Pass@1 on HumanEval for DS-Coder-6.9B $\rightarrow$ DS-Coder-1.3B.
}
\label{tab:heval_mbpp_ds_reduced}
\vspace{-1em}
\end{table}

We further examine how this weighting changes the distribution of
optimization magnitude. Over 480,361 valid token positions, we compute
per-token gradients with respect to the student logits and compare the
frequency of positions with $w_t\geq1$ against their share of the total
weighted gradient norm. As shown in Table~\ref{tab:gradient_allocation},
such positions contribute disproportionately more optimization magnitude
than their frequency across both domains and both prefix branches.
For example, only 19.72\% of code teacher-prefix positions satisfy
$w_t\geq1$, yet they account for 40.43\% of the weighted gradient-norm
mass; under the code student-prefix branch, 28.12\% of positions account
for 46.08\%. This analysis clarifies the role of amplification: validity weighting
reallocates learning-signal magnitude across reasoning states. Whether the
probability of the sampled student action itself increases depends on the
complete teacher and student distributions and is not guaranteed by
$w_t>1$.
\subsection{PRM-Free Validity Approximation}
\label{sec:prm_free}
\subsubsection{PRM‑Free Validity Approximation Setup} \label{prmfree}
For teacher-student pairs without an available process reward model (e.g., DeepSeek-Coder), we approximate local validity using the teacher's next-token likelihood under the shared prefix. The student proposes a greedy next token $a_t^S = \arg\max_a p_s(a \mid c_t)$, and we use the teacher-assigned probability $p_t(a_t^S \mid c_t)$ as the raw student reward. To obtain a stable teacher baseline, we take the top-$k$ teacher probabilities with $k=128$ and compute a collision-based score $\sum_{i=1}^{k} p_t(i \mid c_t)^2$. The local validity weight is then defined as
\[
w_t
=
\frac{p_t(a_t^S \mid c_t) }{\sum_{i=1}^{k} p_t(i\mid c_t)^2 + \varepsilon},
\qquad \varepsilon = 10^{-8},
\]
followed by log-space smoothing and symmetric clamping: we compute $\log w_t$, scale by a factor $\gamma = 0.5$, and clip to $[\log 0.5, \log 2.0]$ before exponentiating back to obtain $w_t \in [0.5, 2.0]$. This weight is applied multiplicatively to both LV--SKL and LV--SRKL losses at each token. 

\subsubsection{PRM‑Free VCRD}
\vspace{- 1 em}
\begin{table}[H]
\centering
\resizebox{1\linewidth}{!}{%
\begin{tabular}{lcccc|c}
\toprule
\textbf{Method}
& GSM8K & MATH & Gaokao & Olympiad & AVG \\
\midrule

Distillm-2
& 85.70 & 73.20 & 61.60 & 34.10 & 63.65 \\
\textbf{VCRD-Prob}

& 84.80 & 73.90 & 64.20 & 33.60 & \textbf{64.13} \\
\bottomrule
\end{tabular}
}
\caption{Pass@1 on four math-reasoning tasks. AVG is the mean across the four tasks. VCRD-Prob represents the PRM-free approach.  Qwen-2.5-Math-7B-Instruct\(\rightarrow\)Qwen-2.5-Math-1.5B}
\label{tab:math_prmfree}
\end{table}
In many practical distillation settings, a suitable process reward model
(PRM) may not be available for the target model family or domain. To
extend VCRD to this setting while preserving its relative
teacher--student allocation principle, we replace the external validity
judge with a teacher-likelihood proxy. At each prefix $c_t$, the student proposes a next token $a_t^S$, and we
use the probability assigned to this token by the teacher as a proxy for
its local support. We compare this quantity with a teacher-derived local
baseline $b_t$, yielding the probability-based weight,
$
    w_t^{\mathrm{prob}}
    =
    \frac{p(a_t^S \mid c_t)}
         {b_t + \epsilon}.
$
The baseline $b_t$ is computed from the teacher's local next-token
distribution, followed by smoothing and bounded rescaling for numerical
stability; full details are provided in Appendix~\ref{prmfree}.

This probability-ratio weight preserves the comparative structure of
VCRD by assigning larger weight when the student's proposal receives
greater support relative to a teacher-derived local baseline. We do not
interpret this proxy as an exact estimator of the full-action
trust-region update; rather, it provides a practical relative signal for
allocating distillation strength when an external validity model is
unavailable.
Although the proxy differs from PRM-based validity scoring, it retains
the same local-competence interpretation. Higher relative support for the
student proposal suggests that the corresponding state is more compatible
with the student's current capability, whereas lower relative support
indicates a larger local gap. 
Because $b_t$ is a distribution-level baseline rather than the
teacher's top-1 probability, $p_t(a_t^S\mid c_t)$ can exceed $b_t$,
allowing $w_t^{\mathrm{prob}}>1$ and thereby preserving the
amplification regime.
Empirically, this PRM-free variant behaves consistently with the
PRM-based method. On code generation (Table~\ref{tab:heval_mbpp_ds_reduced}),
VCRD-Prob achieves an average Pass@1 of 42.70, outperforming
Distillm-2 (40.55). On mathematical reasoning
(Table~\ref{tab:math_prmfree}), VCRD-Prob reaches 64.13 compared with
63.65 for Distillm-2. These results indicate that VCRD's relative
learning-signal allocation principle can remain effective even when an
external validity judge is unavailable.

\begin{table*}[ht]
\centering
\resizebox{\linewidth}{!}{%
\begin{tabular}{lccccccc|c}
\toprule
\multicolumn{9}{c}{\textbf{Qwen2.5-Math-7B-Instruct ($\mathcal{M}_T$) $\rightarrow$ Qwen2.5-Math-1.5B ($\mathcal{M}_S$)}} \\
\midrule
\textbf{Method}
& \textbf{GSM8K}
& \textbf{MATH}
& \textbf{Gaokao}
& \textbf{Olympiad}
& \textbf{CMATH}
& \textbf{AIME24}
& \textbf{AMC23}
& \textbf{AVG} \\
\midrule

LV-Joint-$r_s$
& 85.6 & 73.8 & 60.5 & 32.7 & 90.0 & 6.7 & 47.5 & 56.68 \\

LV-Joint-($r_s-r_t$)
& 85.2 & 73.6 & 61.8 & 34.2 & 89.2 & 10.0 & 52.5 & 58.03 \\

\textbf{VCRD} 
& 85.0 & 73.7 & 62.1 & 35.4 & 90.2 & 13.3 & 60.0
& \textbf{59.96} \\

\bottomrule
\end{tabular}
}
\caption{
Ablation of validity-weighting rules on Qwen2.5-Math-7B-Instruct $\rightarrow$ Qwen2.5-Math-1.5B.
AVG is the average over all datasets.
}
\label{tab:reward_type_ablation}
\end{table*}
\vspace{-0.5 em}
\section{Ablation Study} \label{Rewardtype ablation}
Table~\ref{tab:reward_type_ablation} further examines how local validity scores should be converted into distillation weights. LV-Joint-$r_s$ weights the loss using only the student's raw validity score $r_s$, ignoring the teacher score under the same prefix. Its lower performance indicates that absolute validity alone is not a reliable allocator of learning signal. LV-Joint-($r_s-r_t$) uses the difference between student and teacher validity scores, which better reflects the local improvement direction and improves over raw $r_s$ weighting. However, it still falls short of VCRD, which uses the relative validity ratio $r_s/(r_t+\varepsilon)$. This ratio compares the student continuation against the teacher continuation under the same context while normalizing for the local scale of the judge scores. 
The stronger performance of the ratio supports its use as a practical,
scale-relative allocation rule; we do not interpret this comparison as
showing that the ratio is uniquely implied by the theoretical analysis in
Section~\ref{sec:theory}.
\section{Computational Cost}
\label{computational cost}
\begin{table}[H]
\centering
\resizebox{1\columnwidth}{!}{%
\begin{tabular}{lc}
\toprule
\textbf{Method} & \textbf{Time (s/iteration)} \\
\midrule
KD & 30.19 \\
Distillm-2 & 32.84 \\
VCRD (PRM-based) & 33.84 \\
VCRD-Prob (PRM-free) & 32.31 \\
\bottomrule
\end{tabular}
}
\caption{
Average training time per iteration on Qwen2.5-Math-7B-Instruct ($\mathcal{M}_T$) $\rightarrow$ Qwen2.5-Math-1.5B ($\mathcal{M}_S$). 
}
\label{tab:runtime_qwen25_math}
\end{table}

\begin{table*}[h]
\centering
\resizebox{\linewidth}{!}{
\begin{tabular}{lccccccc}
\toprule
\textbf{PRM} & \textbf{GSM8K} & \textbf{MATH} & \textbf{SVAMP} & \textbf{Minerva} & \textbf{Gaokao} & \textbf{Olympiad} & \textbf{AVG} \\
\midrule
VCRD (Skywork-o1-Open-PRM-1.5B ) & 81.0 & 61.8 & 85.9 & 25.7 & 53.2 & 24.7 & 55.38 \\
VCRD (Qwen2.5-Math-PRM-7B)     & 81.6 & 60.8 & 85.5 & 27.2 & 53.5 & 24.3 & 55.48 \\
\midrule
SKL+SRKL & 81.6 &	61.0 &	85.9&	25.7&	51.4&	21.9	&54.58 \\
SRKL (on-policy KD) &78.0&57.6&83.6&17.3&47.5&21.9&50.98\\

\bottomrule
\end{tabular}
}
\caption{
Sensitivity to PRM choice in the
Qwen2-Math-7B-Instruct ($M_T$)
$\rightarrow$ Qwen2-Math-1.5B ($M_S$) setting.
AVG is the mean across the six benchmarks.
}
\label{tab:prm-sensitivity-qwen2}
\end{table*}

\begin{table*}[!t]
\centering
\small
\resizebox{\linewidth}{!}{
\begin{tabular}{lcccccccc}
\toprule
$\rho$ & GSM8K & MATH & Gaokao & Olympiad & CMATH & AIME24 & AMC23 & AVG \\
\midrule
0.0 & 85.0 & 73.7 & 62.1 & 35.4 & 90.2 & 13.3 & 60.0 & 59.96 \\
0.1 & 85.1 & 74.0 & 62.6 & 35.4 & 89.5 & 13.3 & 57.5 & 59.63 \\
0.3 & 85.4 & 73.3 & 61.3 & 33.0 & 89.5 & 13.3 & 47.5 & 57.61 \\
0.5 & 84.4 & 73.8 & 61.6 & 32.7 & 88.8 & 3.3 & 50.0 & 56.37 \\
\bottomrule
\end{tabular}
}
\caption{Sensitivity of VCRD to corruption of the relative
teacher--student validity ordering in the
Qwen2.5-Math-7B-Instruct$\rightarrow$Qwen2.5-Math-1.5B setting.
At a fraction $\rho$ of valid token positions, $r_T$ and $r_S$ are
randomly swapped before computing the weight.}
\label{tab:validity_corruption}
\end{table*}

We report the training-time cost of VCRD in Table~\ref{tab:runtime_qwen25_math}. Following Distillm-2, all methods use a shared offline data-preparation stage in which teacher and/or student trajectories are generated once before training and then kept fixed. This isolates the comparison to the distillation objective itself, rather than differences in rollout generation. Therefore, the reported times exclude this shared offline stage and measure only the training iteration cost. VCRD introduces only a small additional overhead over existing distillation baselines. In the same Qwen2.5-Math setting, KD requires 30.19 s/iteration, Distillm-2 requires 32.84 s/iteration, and PRM-based VCRD requires 33.84 s/iteration. This shows that the extra PRM-based local-validity computation adds only about 1 second per iteration over Distillm-2. The PRM-free variant has essentially the same cost as standard distillation, requiring 32.31 s/iteration. 
\section{Robustness to the Validity Estimator}
\label{app:validity_robustness}

\paragraph{Sensitivity to PRM scale.}
Table~\ref{tab:prm-sensitivity-qwen2} compares
Skywork-o1-Open-PRM-1.5B with the larger
Qwen2.5-Math-PRM-7B in the same distillation setting.
The two PRMs yield similar average scores
(55.38 and 55.48), and both outperform the
non-calibrated SKL+SRKL and SRKL baselines.
These results support robustness to PRM choice
in this setting and suggest that a substantially
larger reward model is not necessary for VCRD
to be effective.

\paragraph{Sensitivity to validity-ordering errors.}
Changing PRM scale does not directly test what happens when the validity
judge makes incorrect relative comparisons. We therefore perform a
controlled corruption experiment in which, at a fraction $\rho$ of valid
token positions, we randomly swap the teacher and student validity scores
$r_T$ and $r_S$ before computing the validity ratio. This intervention
preserves the original score values and marginal score distribution while
selectively corrupting the relative teacher-student ordering.
As shown in Table~\ref{tab:validity_corruption}, VCRD is relatively insensitive to mild corruption: at $\rho=0.1$, the
average decreases by only 0.33 points, from 59.96 to 59.63. As the
relative ordering is increasingly corrupted, however, performance
declines to 57.61 at $\rho=0.3$ and 56.37 at $\rho=0.5$. This gradual
degradation indicates that informative teacher-student validity ordering
contributes meaningfully to VCRD's gains, while also showing tolerance to
moderate errors from the validity estimator.

\end{document}